\documentclass[12pt]{article}
\usepackage{amsmath}
\usepackage{url} 
\usepackage{amssymb}
\usepackage{amsfonts}
\usepackage{amsthm}
\usepackage{algorithmic}
\usepackage{algorithm}
\usepackage{booktabs}
\usepackage{multirow}
\usepackage{IEEEtrantools}
\usepackage{graphicx,color}
\usepackage[bf,SL,BF]{subfigure}
\usepackage{natbib}
\usepackage{color,xcolor}
\usepackage[colorlinks=true, citecolor=blue, linkcolor=blue]{hyperref}
\usepackage{times}
\usepackage{enumerate}
\usepackage[title]{appendix}
\newcommand{\bPsi}{\boldsymbol{\Psi}}\newcommand{\bS}{\mathbf{S}}\newcommand{\bM}{\mathbf{M}}
\newcommand{\bX}{\mathbf{X}}\newcommand{\bU}{\mathbf{U}}\newcommand{\bI}{\mathbf{I}}\newcommand{\bLmd}{\boldsymbol{\Lambda}}\newcommand{\bZ}{\mathbf{Z}}\newcommand{\bo}{\mathbf{0}}\newcommand{\bx}{\mathbf{x}}\newcommand{\bz}{\mathbf{z}}\newcommand{\by}{\mathbf{y}}\newcommand{\bu}{\mathbf{u}}\def\b1{{\mathbf1}}

\newcommand{\bbE}{\mathbb{E}}

\newcommand{\btSig}{\widetilde{\boldsymbol{\Sigma}}}\newcommand{\btmu}{\widetilde{\bmu}}\newcommand{\btM}{\widetilde{\bM}}

\newcommand{\cL}{\mathcal{L}}\newcommand{\cT}{\mathcal{T}}\newcommand{\cN}{\mathcal{N}}\newcommand{\cX}{\mathcal{X}}\newcommand{\cW}{\mathcal{W}}

\newcommand{\bwSig}{\widehat{\bSig}}

 \newcommand{\bttheta}{\tilde{\btheta}}

\newcommand{\bmu}{\boldsymbol{\mu}}\newcommand{\btheta}{\boldsymbol{\theta}} 
\newcommand{\bSig}{\boldsymbol{\Sigma}}

\newcommand\refe[1]{(\ref{#1})}\newcommand\refp[1]{Proposition~\ref{#1}}\newcommand\refap[1]{\ref{#1}}
\newcommand\reff[1]{Fig.~\ref{#1}}\newcommand\reft[1]{Tab.~\ref{#1}}\newcommand\refs[1]{Sec.~\ref{#1}}
\newcommand\refa[1]{Algorithm~\ref{#1}}  
 \newcommand\refdef[1]{Definition~\ref{#1}}
\theoremstyle{plain} \newtheorem{prop}{Proposition}\newtheorem{defi}{Definition} 

\def\R{{\mathbb R}}

\def\tr{\mbox{tr}}
\def\vc{\mbox{vec}}
\def\cov{\mbox{cov}}
\def\diag{\mbox{diag}}

\def\s2{\sigma^2}
\def\Gam{\mbox{Gam}}

\newcommand{\blind}{0}

\begin{document}

\def\spacingset#1{\renewcommand{\baselinestretch}%
	{#1}\small\normalsize} \spacingset{1}


\if0\blind
{
	\title{\bf Robust factored principal component analysis for matrix-valued outlier accommodation and detection}
	\author{Xuan Ma, Jianhua Zhao\thanks{The corresponding author. Email: jhzhao.ynu@gmail.com}\hspace{.2cm} and Yue Wang \\
	School of Statistics and Mathematics, Yunnan University of Finance and Economics}
	\maketitle
} \fi
\if1\blind
{
	\bigskip
	\bigskip
	\bigskip
	\begin{center}
		{\LARGE\bf Robust factored principal component analysis for matrix-valued outlier accommodation and detection}
	\end{center}
	\medskip
} \fi

\bigskip

\begin{abstract}
Principal component analysis (PCA) is a popular dimension reduction technique for vector data. Factored PCA (FPCA) is a probabilistic extension of PCA for matrix data, which can substantially reduce the number of parameters in PCA while yield satisfactory performance. However, FPCA is based on the Gaussian assumption and thereby susceptible to outliers. Although the multivariate \emph{t} distribution as a robust modeling tool for vector data has a very long history, its application to matrix data is very limited. The main reason is that the dimension of the vectorized matrix data is often very high and the higher the dimension, the lower the breakdown point that measures the robustness. To solve the robustness problem suffered by FPCA and make it applicable to matrix data, in this paper we propose a robust extension of FPCA (RFPCA), which is built upon a \emph{t}-type distribution called matrix-variate \emph{t} distribution. Like the multivariate \emph{t} distribution, the matrix-variate \emph{t} distribution can adaptively down-weight outliers and yield robust estimates. We develop a fast EM-type algorithm for parameter estimation. Experiments on synthetic and real-world datasets reveal that RFPCA is compared favorably with several related methods and RFPCA is a simple but powerful tool for matrix-valued outlier detection.
\end{abstract}

\spacingset{1.75} 
\section{Introduction}\label{sec:intr}
Principal component analysis (PCA) \citep{jolliffe} is a popular dimension reduction technique for vector data (where observations are vectors) and widely used in data compression, pattern recognition, computer vision and signal processing \citep{kirby-eigenface}. PCA is obtained from the eigen-decomposition of the sample covariance matrix. To apply PCA to matrix data (where observations are matrices of dimension $c\times r$) such as images, the matrices must be first vectorized as vectors. However, vectorization will destroy the natural matrix structure and may lose useful local information between rows or columns \citep{jpye-glram}. Moreover, the vectorized data dimension $d=cr$ is usually very high \citep{james-mvfa} while the number of parameters in PCA increases linearly with $d$. When $d$ is large and the sample size is relatively low, PCA is easily trapped into the \emph{curse of dimensionality}. 

To solve the high-dimensional problem, several matrix-based dimension reduction methods have been proposed. For example, bidirectional PCA (BPCA) \citep{zhangdq-2dpca}, generalized low rank approximations of matrices (GLRAM) \citep{jpye-glram} and factored PCA (FPCA) \citep{dryden-sbpca}. The common feature is that these methods use bilinear transformations, rather than the linear one in PCA. Their numbers of parameters increase linearly with the row dimension $c$ and column dimension $r$ and hence are much less than that in PCA. As a result, these methods can alleviate greatly the high-dimensional problem and significantly reduce computation costs \citep{jpye-glram}. FPCA is based on the assumption of the matrix-variate normal (matrix-normal) distribution, which is essentially the multivariate normal with a factored (or separable) covariance structure. Compared with BPCA and GLRAM, FPCA is a probability method that is able to handle the uncertainties from different sources by means of probability theory \citep{dryden-sbpca}. However, the normal assumption is sensitive to the departure of normality and may result in biased or even misleading results when data involves heavy tails or outliers \citep{lange-mvt,lin-stfa}.

For low-dimensional vector data, it has been proved that the multivariate \emph{t} distribution is effective for robust estimation \citep{lange-mvt}. Naturally, a robust PCA can be obtained by the robust covariance matrix from a multivariate \emph{t} distribution (\emph{t}PCA, in short). The extensions with the multivariate \emph{t} distribution to related models such as probabilistic PCA and factor analysis can be found in \cite{zhao2006-rpca-t,mclachlan2007extension}. 
However, the multivariate \emph{t} distribution is only applicable to vector data. When it is applied to matrix data, the dimension $d$ of the vectorized matrix data is usually very high while the upper bound of its breakdown point, that measures the robustness, is $1/d$ \citep{Dumbgen-bkd}. Consequently, the larger the value of $d$, the lower the breakdown point. This greatly limits its applicability to matrix data. 

To improve the fitting for matrix data, \cite{thompson2020classification} suggest using a \emph{t}-type distribution, which is called the matrix-variate \emph{T} (matrix-\emph{T}) distribution in \cite{gupta2013elliptically}, to replace the matrix-normal distribution, and their experimental results show that the model with a matrix-\emph{T} distribution yields better performance in classification applications. However, they fail to consider the robustness of their model. It is not clear how robust their model is in comparison with the models with the matrix-normal and particularly, multivariate \emph{t} distributions. This is one of our main concerns in this paper. For clarity, the FPCA using a matrix-\emph{T} distribution is denoted by TPCA in this paper.

To solve the robustness problems suffered by FPCA and \emph{t}PCA and make the resulting method applicable to matrix data. In this paper we propose a new robust FPCA method for matrix data, denoted as RFPCA. RFPCA uses another \emph{t}-type distribution, which is called the matrix-variate \emph{t} (matrix-\emph{t}) distribution in \cite{gupta2013elliptically}. Furthermore, we develop a fast Expectation Maximization (EM)-type algorithm to obtain the maximum likelihood estimates of model parameters. The main contributions of this paper are threefold as follows.
\begin{enumerate}[(i)]
	\item RFPCA is generally comparable with or more robust and accurate than related methods. Moreover, with our proposed fast algorithm, RFPCA can be run much faster than TPCA, especially when the sample size gets large, as will be seen in \refs{sec:eff}.  
	\item More importantly, while a rigorous theoretical result is not yet available, our empirical results show that RFPCA has a significantly higher breakdown point than its vector-based counterpart \emph{t}PCA, which makes RFPCA applicable to matrix data. To our knowledge, this important result is found for the first time. 
	\item As will be seen in \refs{sec:face.odtc}, the expected latent weights of RFPCA can be readily used for outlier detection, and they are much more reliable than those by \emph{t}PCA, due to the better robustness. In contrast, it is not clear how outlier detection can be performed for TPCA. This means that RFPCA provides a simple but powerful tool to detect potential outlying matrix observations. Furthermore, compared with the freshly proposed matrix-variate contaminated normal distribution that can only be used to detect mild outliers \citep{tomarchio2021mixtures}, RFPCA is inherently able to detect gross outliers (i.e., arbitrarily large abnormal observations), due to the use of the matrix-\emph{t} distribution. To our knowledge, our method provides the first matrix-based tool for gross matrix-valued outlier detection.
\end{enumerate}

The rest of this paper is organized as follows. In \refs{sec:rlt}, we briefly review some related works. In \refs{sec:MtPCA}, we propose RFPCA and develop its maximum likelihood estimation algorithms. In \refs{sec:expr}, we conduct a number of experiments on simulated and real-world data. In \refs{sec:Conclusion}, we close the paper with some conclusions and discussions. 

{\bf Notations}: In this paper, $\bI_c$ is an identity matrix of dimension $c \times c$. The transpose of vector $\bx$ or matrix $\bX$ is denoted by $\bx'$ or $\bX'$, the matrix trace by $\tr(\cdot)$, the vectorization by $\vc(\cdot)$ and the Kronecker product by $\otimes$.  

\section{Related works}\label{sec:rlt}
In this section, we briefly review the multivariate \emph{t}, matrix-variate normal and matrix-variate \emph{T} distributions and the PCA variants with these distributions. 

\subsection{Multivariate \emph{t} distribution and multivariate-\emph{t} PCA (\emph{t}PCA) }\label{sec:tPCA}
We first recall some fundamental results on the Gamma and Wishart distributions \citep{bishop-prml}.

Suppose that a random variable $\tau$ follows the Gamma distribution $\Gam(\alpha,\beta)$, i.e., $\tau\sim\Gam(\alpha,\beta)$. The expectations of $\tau$ and $\ln{\tau}$ are given by 
\begin{IEEEeqnarray}{rCl}
	\bbE[\tau]&=& \frac{\alpha}\beta,\label{eqn:gam.etau}\\
	\bbE[\ln{\tau}] &=& \psi{(\alpha)}-\ln{\beta}, \label{eqn:gam.elntau}
\end{IEEEeqnarray}
where $\psi(x)=d\ln(\Gamma(x))/dx$ is the digamma function. 

Suppose that a random matrix $\bS$ follows the Wishart distribution $\cW_{d}(\bPsi,\nu)$ with the scale matrix $\bPsi$ and degrees of freedom $\nu$, i.e., $\bS\sim\cW_{d}(\bPsi,\nu)$. The expectations of $\bS$ and $\ln{|\bS|}$ are 
\begin{IEEEeqnarray}{rCl}
	\bbE[\bS]&=& \nu\bPsi,\label{eqn:wst.eS}\\
	\bbE[\ln{|\bS|}] &=& \psi_d{(\nu/2)}+d\ln{2}+\ln{|\bPsi|}, \label{eqn:wst.elnS}
\end{IEEEeqnarray}
where $\psi_d(x)=d\ln(\Gamma_{d}(x))/dx$ is the $d$-variate digamma function.

\subsubsection{Multivariate \emph{t} distribution}\label{sec:t}
Suppose that a $d$-dimensional random vector $\bx$ follows the multivariate \emph{t} distribution with center $\bmu\in \R^{d}$, positive definite matrix $\bSig\in \R^{d\times d}$ and degrees of freedom $\nu>0$, denoted by $t_d(\bmu,\bSig,\nu)$, then the probability density function (p.d.f.) of $\bx$ is given by
\begin{equation}\label{eqn:t.pdf}
	\hskip-1em p(\bx)=\frac{|\bSig|^{-\frac{1}{2}} \Gamma(\frac{\nu+d}{2})}{(\pi \nu)^{\frac{d}{2}}\Gamma(\frac{\nu}{2})}
	\left(1+\frac{1}{\nu}(\bx-\bmu)'\bSig^{-1}(\bx-\bmu)\right)^{-\frac{\nu+d}{2}}.
\end{equation}
If $\nu>1$, $\bmu$ is the mean, and if $\nu>2$, $\bSig\nu/(\nu-2)$ is the covariance matrix \citep{liu-tdist}. 

$\bx$ can also be represented hierarchically as two different latent variable models \citep{gupta-mvn}.

(i) \textsf{Model I}: Introduce a latent variable $\tau$ distributed as the Gamma distribution $\Gam(\nu/2,\nu/2)$, and given $\tau$, $\bx$ conditionally follows the multivariate normal distribution $\cN_d(\bmu,\bSig/\tau)$, that is, 
\begin{equation}\label{eqn:t.hrc1}
	\textsf{Model I}: \quad \bx|\tau\sim\cN_d(\bmu,\bSig/\tau),\,\, \tau\sim\Gam(\nu/2,\nu/2).
\end{equation}
Under \textsf{Model I}, we have the marginal distribution $\bx\sim t_d(\bmu,\bSig,\nu)$, and the posterior distribution
\begin{equation}
	\tau|\bx\sim\Gam\left(\frac{\nu+d}{2},\frac{\nu+(\bx-\bmu)'\bSig^{-1}(\bx-\bmu)}{2}\right).
	\label{eqn:pdf.tau|x}
\end{equation}  

(ii) \textsf{Model II}: Introduce a random matrix $\bS$ distributed as the Wishart distribution $\cW_d(\bSig^{-1}, \nu+d-1)$. Given $\bS$, $\bx$ conditionally follows the multivariate normal distribution $\cN_d(\bmu,\nu\bS^{-1})$, i.e.,
\begin{equation}\label{eqn:t.hrc2}
	\textsf{Model II}: \quad \bx|\bS\sim\cN_d(\bmu,\nu\bS^{-1}),\,\, \bS\sim\cW_d(\bSig^{-1}, \nu+d-1).
\end{equation}
It can be shown that the marginal density of $\bx$ from $\int{p(\bx|\bS)p(\bS)}d\bS$ is also the p.d.f. of $\bx$ in \refe{eqn:t.pdf}.

\subsubsection{Multivariate-\emph{t} PCA (\emph{t}PCA)}\label{sec:t.tPCA}
Consider the orthogonal linear transformation $\by=\bU'\bx$, where $\bU\in\R^{d\times q}$ and $q<d$. Let $\bSig$ be the covariance matrix of $\bx$. Then the covariance matrix of $\by$ is $\bU'\bSig\bU$. Principal component analysis (PCA) aims to find $\bU$ that maximizes the covariance in $\by$-space. This can be formulated as the optimization problem \citep{jolliffe} 
\begin{equation}
	\mathop{\hbox{max}}\nolimits_\bU\hbox{tr}\left\{\bU'\bSig\bU\right\},\,\,s.t.\, \bU'\bU=\bI_q.\label{eqn:pca.obj}
\end{equation}
Let $\{(\bu_i,\lambda_i)\}_{i=1}^{d}$ be the decreasing-ordered eigenpairs of $\bSig$. The closed form solution $\bU$ to \refe{eqn:pca.obj} is 
\begin{equation}
	\bU=[\bu_1,\bu_2,\dots,\bu_q],\hbox{top}\ q\ \hbox{eigenvectors of}\ \bSig,\label{eqn:pca.U}
\end{equation}
For convenience, let $\bLmd=\diag\{\lambda_1,\lambda_2,\dots,\lambda_q\}$ be the diagonal matrix formed by the corresponding eigenvalues. Under the assumption of the multivariate normal distribution, $\bSig$ is estimated by the sample covariance matrix. 

Under the assumption of the multivariate \emph{t} distribution, $\bSig$ in \refe{eqn:pca.U} can be estimated by the robust sample covariance matrix $\btSig$ given in \refe{eqn:ecme.t.S}. We call this variant multivariate-\emph{t} PCA (\emph{t}PCA). 

\subsubsection{Maximum likelihood estimation of the multivariate \emph{t} distribution}\label{sec:t.mle}
Given data $\cX=\{\bx_n\}_{n=1}^N$, the ML estimate for the multivariate \emph{t} distribution can be obtained by EM algorithm or its several variants. Below we give a sketch of a variant called expectation conditional maximization either (ECME) algorithm \citep{chuanhai_ecme}. The ECME for the multivariate \emph{t} distribution alternates the following one E-step and three CM-steps \citep{liu-tdist}.

\noindent{\bf E-step}: Treat $\cT=\{\tau_n\}_{n=1}^N$ as missing data. The log likelihood function of the complete data is  $\sum\nolimits_{n=1}^N\ln{ \{p(\bx_n|\tau_n)p(\tau_n)\}},$  and its expectation with respect to (w.r.t.) the conditional distributions $p(\tau_n|\bx_n,\btheta)$ yielding the $Q$ function (up to a constant),
\begin{IEEEeqnarray}{rCl}\label{eqn:tpca.Q}
	Q(\btheta|\cX)&=&-\frac1{2}\sum\nolimits_{n=1}^N\left\{\ln{|\bSig|}+\bbE[\tau_n|\bx_n](\bx_n-\bmu)'\bSig^{-1}(\bx_n-\bmu)\right\} \nonumber\\ 
	&& +\frac{N\nu}{2}\ln{(\frac{\nu}{2})}-N\ln{\left(\Gamma(\frac{\nu}{2})\right)} +\frac{\nu}{2}\sum\nolimits_{n=1}^N\left\{\ln{(\bbE[\tau_n|\bx_n])}-\bbE[\tau_n|\bx_n]\right\},
\end{IEEEeqnarray}
where, by \refe{eqn:pdf.tau|x} and \refe{eqn:gam.etau}, the expected weight of $\bX_n$ is computed by
\begin{equation}\label{eqn:Etau.x}
	\bbE[\tau_n|\bx_n]= \frac{\nu+d}{\nu+(\bx_n-\bmu)'\bSig^{-1}(\bx_n-\bmu)}.
\end{equation}
\noindent
{\bf CMQ-step 1:} Given $(\bSig, \nu)$, maximizing $Q$ in \refe{eqn:tpca.Q} w.r.t. $\bmu$, we get 
\begin{equation}\label{eqn:ecme.t.mu}
	\btmu=\frac1{\sum\nolimits_{n=1}^N\bbE[\tau_n|\bx_n]}\sum\nolimits_{n=1}^N\bbE[\tau_n|\bx_n]\bx_n. 
\end{equation}
\noindent
{\bf CMQ-step 2:} Given $(\btmu,\nu)$, maximizing $Q$ in \refe{eqn:tpca.Q} w.r.t. $\bSig$, we obtain
\begin{equation}\label{eqn:ecme.t.S}
	\btSig = \frac{1}{N}\sum\nolimits_{n=1}^N\bbE[\tau_n|\bx_n] (\bx_n-\btmu)(\bx_n-\btmu)'.
\end{equation}
\noindent
{\bf CML-step 3:} Given $(\btmu,\btSig)$, the solution of $\tilde{\nu}$ can be obtained by maximizing the observed data log likelihood $\cL=\sum\nolimits_{n=1}^N\ln{p(\bx_n)}$ w.r.t. $\nu$. 

The multivariate \emph{t} distribution has two appealing characteristics: (i) The degrees of freedom $\nu$ is able to tune the robustness. (ii) The expected weight $\bbE[\tau_n|\bx_n]$ in \refe{eqn:Etau.x} can be used to detect outliers. We give more details related to the two characteristics in \refs{sec:mvt.odtc}.

\subsection{Matrix-variate normal distribution and factored PCA (FPCA) }\label{sec:fpca}
\subsubsection{Matrix-variate normal distribution}\label{sec:mvn}
\begin{defi}\label{def.mvn}
	The random matrix $\bX\in\R^{c\times r}$ is said to follow a matrix-variate normal distribution with mean matrix $\bM\in\R^{c\times r}$, column and row covariance matrices $\bSig_c\in\R^{c\times c}$ and $\bSig_r\in\R^{r\times r}$, respectively, denoted by $\bX\sim \cN_{c,r}(\bM, \bSig_c, \bSig_r)$, if $\vc(\bX)\sim \cN_{c\times r}(\vc(\bM), \bSig_r\otimes\bSig_c)$. The p.d.f. of $\bX$ is given by
	\begin{equation}\label{eqn:mvn.density}
		p(\bX)=(2\pi)^{-\frac{cr}{2}}|\bSig_c|^{-\frac{r}{2}}|\bSig_r|^{-\frac{c}{2}}\exp{\left\{-\frac{1}{2}\tr\left(\bSig_c^{-1}(\bX-\bM)\bSig_r^{-1}(\bX-\bM)'\right)\right\}}.
	\end{equation}
\end{defi}
From \refdef{def.mvn}, the matrix-normal distribution is essentially a multivariate normal distribution with a separable covariance matrix $\bSig=\bSig_r\otimes \bSig_c$. 

\subsubsection{Factored PCA (FPCA)}\label{sec:mn.fpca}
Under the assumption of the matrix-normal distribution, the covariance matrix is factorized into $\bSig=\bSig_r\otimes\bSig_c$. Substituting this factorization into \refe{eqn:pca.obj}, we find that the $\bU$ solution is now factorized into $\bU=\bU_r\otimes\bU_c$. Consequently, the objective function of PCA in \refe{eqn:pca.obj} degenerates into two separate sub-functions, each of which is a similar optimization problem to that in PCA,
\begin{equation}
	\mathop{\hbox{max}}\nolimits_{\bU_c}\hbox{tr}\left\{\bU'_c\bSig_c\bU_c\right\},\,\,s.t.\, \bU'_c\bU_c=\bI_{q_c},\quad \mbox{and}\quad \mathop{\hbox{max}}\nolimits_{\bU_r}\hbox{tr}\left\{\bU'_r\bSig_r\bU_r\right\},\,\,s.t.\, \bU'_r\bU_r=\bI_{q_r},\label{eqn:fpca.obj}
\end{equation}
Let $\{(\bu_{ci},\lambda_{ci})\}_{i=1}^c$ and $\{(\bu_{ri},\lambda_{ri})\}_{i=1}^r$ be the decreasing-ordered eigenpairs of $\bSig_c$ and $\bSig_r$, respectively. The closed form solutions $\bU_c$ and $\bU_r$ to \refe{eqn:fpca.obj} are given by
\begin{equation}
	\bU_c=[\bu_{c1},\dots,\bu_{cq_c}],\,\,\bU_r=[\bu_{r1},\dots,\bu_{rq_r}],\, \hbox{top}\ q_c,q_r\ \hbox{eigenvectors of}\ \bSig_c\ \hbox{and}\ \bSig_r,\label{eqn:fpca.U}
\end{equation}
Let $\bLmd_c=\diag\{\lambda_{c1},\lambda_{c2},\dots,\lambda_{cq_c}\}$ and $\bLmd_r=\diag\{\lambda_{r1},\lambda_{r2},\dots,\lambda_{rq_r}\}$ be the diagonal matrices formed by the corresponding eigenvalues. Under the assumption of the matrix-normal distribution, $\bSig_c$ and $\bSig_r$ can be estimated by $\btSig_c$ in \refe{eqn:fpca.Sc} and $\btSig_r$ in \refe{eqn:fpca.Sr}, respectively.

\subsubsection{Maximum likelihood estimation of the matrix-variate normal distribution}\label{sec:mn.mle}
Let $\btheta=(\bM, \bSig_c, \bSig_r)$. Given a dataset $\cX=\{\bX_n\}_{n=1}^N$ , then the observed data log likelihood $\cL$ is, up to a constant, 
\begin{equation*}
	\hskip-0.5em\cL(\btheta|\cX)=-\frac12\sum\nolimits_{n=1}^N\left\{r\ln{|\bSig_c|}+c\ln{|\bSig_r|}+\tr\{\bSig_c^{-1}(\bX_n-\bM)\bSig_r^{-1}(\bX_n-\bM)'\}\right\}.\label{eqn:fpca.like}
\end{equation*}
The global ML estimate of $\bM$ is obviously the sample mean $\bar{\bX}=\frac1N\sum\nolimits_{n=1}^N\bX_n$. The ML estimate of $\bSig_c$ and $\bSig_r$ can be obtained by iterative maximization of $\cL$ using a conditional maximization (CM) algorithm \citep{chuanhai_ecme}, which iterates the following two CM-steps.

\noindent{\bf CM-step 1}: Given $\bSig_r$ and $\bM=\bar{\bX}$, maximizing $\cL$ w.r.t. $\bSig_c$, we have
\begin{equation}
	\btSig_c =  \frac{1}{Nr}\sum\nolimits_{n=1}^N(\bX_n-\bar{\bX})\bSig_r^{-1}(\bX_n-\bar{\bX})'.\label{eqn:fpca.Sc}
\end{equation}

\noindent{\bf CM-step 2}: Given $\btSig_c$ and $\bM=\bar{\bX}$, maximizing $\cL$ w.r.t. $\bSig_r$ yields
\begin{equation}
	\btSig_r =  \frac{1}{Nc}\sum\nolimits_{n=1}^N(\bX_n-\bar{\bX})'\btSig_c^{-1}(\bX_n-\bar{\bX}).\label{eqn:fpca.Sr}
\end{equation}

The computational complexity analysis of the CM is given below. In CM-step 1, calculation of $\bSig_r^{-1}$ costs $O(r^3)$, and $(\bX_n-\bar{\bX})\bSig_r^{-1}(\bX_n-\bar{\bX})'$ costs $O(cr^2+c^2r)$. The total cost for $\btSig_c$ in \refe{eqn:fpca.Sc} is hence $O(N[cr^2+c^2r])+O(r^3)$. Similarly, computation of $\btSig_r$ in CM-step 2 costs $O(N[c^2r+cr^2])+O(c^3)$. Summing over all number of iterations $t$ yields the total cost $O(tN[c^2r+cr^2])+O(t[c^3+r^3])$.

\subsection{Matrix-variate \emph{T} distribution and matrix-\emph{T} PCA (TPCA)}\label{sec:TPCA}

\subsubsection{Matrix-variate \emph{T} distribution (matrix-\emph{T})}\label{sec:matrix-T}
\begin{defi}\label{def:mxvt-1}
	The random matrix $\bX\in\R^{c\times r}$ is said to follow a matrix-variate \emph{T} (matrix-T)
	distribution with center matrix $\bM\in\R^{c\times r}$, positive matrices $\bSig_c\in\R^{c\times c}$, $\bSig_r\in\R^{r\times r}$ and $\nu>0$, denoted by $\bX\sim T_{c, r}\left(\bM, \bSig_c, \bSig_r, \nu\right)$, if its p.d.f. is 
	\begin{IEEEeqnarray*}{rCl}
		p\left(\bX\right)&=& \frac{\left|\bSig_r\right|^{-\frac{c}{2}}\left|\bSig_c\right|^{-\frac{r}{2}} \Gamma_{c}\left(\frac{\nu+c+r-1}{2}\right)}{(\pi)^{\frac{c r}{2}} \Gamma_{c}\left(\frac{\nu+c-1}{2}\right)} \times\left|\bI_{c}+\bSig_c^{-1}(\bX-\bM) \bSig_r^{-1}(\bX-\bM)'\right|^{\frac{\nu+c+r-1}{2}}.\label{eqn:T.pdf}
	\end{IEEEeqnarray*}
\end{defi}
$\bX$ can be expressed as a latent variable model. Given a random matrix $\bS\sim\cW_{c}(\bSig_c^{-1},\nu+c-1)$, the matrix observation $\bX$ follows the matrix-normal distribution $\cN_{c,r}(\bM, \bS^{-1}, \bSig_r)$, that is 
\begin{equation}\label{eqn:T.hrc}
	\textsf{Model III}: \quad \bX|\bS\sim\cN_{c,r}(\bM, \bS^{-1}, \bSig_r),\,\, \bS\sim\cW_{c}(\bSig_c^{-1},\nu+c-1).
\end{equation}
Under \textsf{Model III} in \refe{eqn:T.hrc}, we can obtain the marginal distribution $\bX\sim T_{c, r}\left(\bM, \bSig_c, \bSig_r, \nu\right)$. The covariance matrix is computed by \citep{gupta-mvn}
\begin{equation}
	\cov{(\vc(\bX))}=\frac1{\nu-2}\bSig_r\otimes\bSig_c,\label{eqn:T.covx}
\end{equation}
Furthermore, we can obtain the posterior distribution \citep{thompson2020classification}\begin{equation}
	\bS|\bX\sim\cW_{c}\left(\left[(\bX-\bM)\bSig_r^{-1}(\bX-\bM)'+\bSig_c\right]^{-1},\nu+c+r-1\right).\label{eqn:T.S|X}
\end{equation} 
By comparing \refe{eqn:T.hrc} and \refe{eqn:t.hrc2}, it can be seen that the matrix-\emph{T} distribution is an extension of the multivariate \emph{t} distribution from \textsf{Model II} in vector cases to the matrix ones, where the multivariate normal distribution is replaced by a matrix-normal distribution.

\subsubsection{Matrix-\emph{T} PCA (TPCA)}\label{sec:matrix-T.TPCA}
Similar to that of FPCA in \refs{sec:mn.fpca}, we introduce matrix-\emph{T} PCA (TPCA) below. The $\bU_c$ and $\bU_r$ solutions are obtained by substituting the factorization $\bSig=\bSig_r\otimes\bSig_c$ in \refe{eqn:T.covx} (up to a constant) under the matrix-\emph{T} distribution assumption into \refe{eqn:fpca.obj} and \refe{eqn:fpca.U}, where $\bSig_c$ and $\bSig_r$ can be estimated by $\btSig_c$ in \refe{eqn:T.Sc} and $\btSig_r$ in \refe{eqn:T.Sr}, respectively, as detailed in \refs{sec:matrix-T.mle}. We call this variant matrix-\emph{T} PCA, denoted by TPCA. 

\subsubsection{Maximum likelihood estimation of the matrix-\emph{T} distribution}\label{sec:matrix-T.mle}
Given a set of observations $\cX=\{\bX_n\}_{n=1}^N$, the ML estimate of the parameter $\btheta=(\bM, \bSig_c, \bSig_r,\nu)$ can be obtained by an ECME algorithm \citep{thompson2020classification}, which consists of the following one E-step and three CM-steps.

\noindent{\bf E-step}: Let the missing data be $\{\bS_n\}_{n=1}^N$, then the complete data log likelihood is $\sum\nolimits_{n=1}^N\ln{\{p(\bX_n|\bS_n)p(\bS_n)\}}$. Compute its expectation w.r.t. the conditional distributions $p(\bS_n|\bX_n)$ yielding
\begin{IEEEeqnarray}{rCl}\label{eqn:Tpca.Q}
	Q(\btheta|\cX)&=&N\left\{\frac{\nu+c-1}{2}\ln{|\bSig_c|}-\frac{c}{2}\ln{|\bSig_r|}-\frac{c\nu}{2}\ln{2}-\ln{\Gamma_{c}\left(\frac{\nu+c-1}{2}\right)}\right\}+\sum\nolimits_{n=1}^N\left\{\frac{r}{2}\bbE[\ln{|\bS_n|\big|\bX_n}]\right.\nonumber\\
	&&\hskip0.5em\left.-\>\frac{1}{2}\tr\{\bbE[\bS_n|\bX_n](\bX_n-\bM)\bSig_r^{-1}(\bX_n-\bM)'\} +\frac{\nu-2}{2}\bbE[\ln{|\bS_n|\big|\bX_n}]-\frac{1}{2}\tr\{\bSig_c\bbE[\bS_n|\bX_n]\}\right\},
\end{IEEEeqnarray}
where, from \refe{eqn:T.S|X}, the required expectations are computed by
\begin{IEEEeqnarray}{rCl}
	\bbE[\bS_n|\bX_n]&=&\left(\nu+c+r-1\right)\left[(\bX_n-\bM)\bSig_r^{-1}(\bX_n-\bM)'+\bSig_c\right]^{-1},\label{eqn:T.ES}\\
	\bbE[\ln{|\bS_n|\big|\bX_n}]&=&\psi_{c}\left(\frac{\nu+c+r-1}{2}\right)+c\ln{2}+\ln{\left|\frac{\bbE[\bS_n|\bX_n]}{\nu+c+r-1}\right|}.\nonumber
\end{IEEEeqnarray}
Here $\psi_{c}(x)=d\ln(\Gamma_{c}(x))/dx$ is the $c$-variate digamma function.

\noindent
{\bf CMQ-step 1:} Given $(\bSig_c,\bSig_r,\nu)$, maximizing $Q$ in \refe{eqn:Tpca.Q} w.r.t. $\bM$, we obtain 
\begin{equation}
	\btM=\left(\sum\nolimits_{n=1}^N\bbE[\bS_n|\bX_n]\right)^{-1}\sum\nolimits_{n=1}^N\bbE[\bS_n|\bX_n]\bX_n.\label{eqn:T.mean}
\end{equation}

\noindent
{\bf CMQ-step 2:} Given $(\btM,\bSig_r,\nu)$, maximize $Q$ in \refe{eqn:Tpca.Q} w.r.t. $\bSig_c$ yielding
\begin{equation}
	\btSig_c^{-1} = \frac{1}{N(\nu+c-1)}\sum\nolimits_{n=1}^N\bbE[\bS_n|\bX_n].\label{eqn:T.Sc}
\end{equation}

\noindent
{\bf CMQ-step 3:} Given $(\btM,\btSig_c,\nu)$, maximize $Q$ w.r.t. $\bSig_r$ leading to
\begin{equation}
	\btSig_r = \frac{1}{Nc}\sum\nolimits_{n=1}^N(\bX_n-\btM)'\bbE[\bS_n|\bX_n](\bX_n-\btM).\label{eqn:T.Sr}
\end{equation}

\noindent
{\bf CML-step 4:} Given $(\btM,\btSig_c,\btSig_r)$, maximize the observed data log likelihood $\cL=\sum\nolimits_{n=1}^N\ln{p( \bX_n)}$ w.r.t. $\nu$. The updated $\tilde{\nu}$ is the root of the equation
\begin{equation*}
	\psi_{c}\left(\frac{\nu+c+r-1}{2}\right)-\psi_{c}\left(\frac{\nu+c-1}{2}\right)+\frac{1}{N} \sum\nolimits_{n=1}^N \ln{\left|\frac{\bbE[\bS_n|\bX_n]}{\nu+c+r-1}\right|}-\ln {\left|\frac{\sum\nolimits_{n=1}^N\bbE[\bS_n|\bX_n]}{N\left(\nu+c-1\right)}\right|}=0.
\end{equation*}
Next we analyze the computational complexity of this algorithm. The main computations lie in E-step, CM-step 1, 3 and 4. In E-step, the computation of $\bbE[\bS_n|\bX_n]$ \refe{eqn:T.ES} involves the inverse of a $c\times c$ matrix, which is $O(c^3)$, and thus the cost of E-step takes $O(N[cr^2+c^2r+c^3])$. In CM-step 1, the cost of $\sum\nolimits_{n=1}^N\bbE[\bS_n|\bX_n]\bX_n$  \refe{eqn:T.mean} is $O(Nc^2r)$. 
In CM-step 3, the computation of $\sum\nolimits_{n=1}^N(\bX_n-\btM)'\bbE[\bS_n|\bX_n](\bX_n-\btM)$ \refe{eqn:T.Sr} takes $O(N[cr^2+c^2r])$. CML-step 4 involves the computation of $\ln{\left|\bbE[\bS_n|\bX_n]\right|}$, which costs $O(Nc^3)$. Summing over all number of iterations $t$ yields the total cost  $O(tN[c^2r+cr^2+c^3])+O(t[c^3+r^3])$. 

\section{Robust FPCA (RFPCA) via the matrix-\emph{t} distribution}\label{sec:MtPCA}
In this section, we propose a robust FPCA (RFPCA) method. In \refs{sec:matrix-t},  we extend \textsf{Model I} \refe{eqn:t.hrc1} of the multivariate \emph{t} distribution to matrix cases and obtain the matrix-variate \emph{t} (matrix-\emph{t}) distribution \citep{gupta2013elliptically}. Based on the matrix-\emph{t} distribution, we propose the RFPCA in \refs{sec:MtPCA.model}. In \refs{sec:matrix-t.mle}, we develop two ML estimation algorithms for parameter estimation of RFPCA.

\subsection{Matrix-variate \emph{t} distribution (matrix-\emph{t})}\label{sec:matrix-t}
\begin{defi}\label{def.mxvt-2}
	The random matrix $\bX\in\R^{c\times r}$ is said to follow a matrix-variate \emph{t} distribution with mean matrix $\bM\in\R^{c\times r}$, column and row covariance matrices $\bSig_c\in\R^{c\times c}$ and $\bSig_r\in\R^{r\times r}$, and degrees of freedom $\nu$, denoted by $\bX\sim Mt_{c,r}(\bM, \bSig_c, \bSig_r, \nu)$, if $\vc(\bX)\sim t_{cr}(\vc(\bM), \bSig_r\otimes\bSig_c,\nu)$. The p.d.f. of $\bX$ is as follows:
	\begin{equation}\label{eqn:mvt.density}
		p(\bX)=\frac{|\bSig_c|^{-\frac{r}{2}}|\bSig_r|^{-\frac{c}{2}} \Gamma(\frac{\nu+cr}{2})}{(\pi \nu)^{\frac{cr}{2}}\Gamma(\frac{\nu}{2})}\left(1+\frac{1}{\nu}\tr\{\bSig_c^{-1}(\bX_n-\bM)\bSig_r^{-1}(\bX_n-\bM)'\}\right)^{-\frac{\nu+cr}{2}}.
	\end{equation}
\end{defi}
$\bX$ can be represented as a latent variable model. Introduce a latent weight variable $\tau$ distributed as the Gamma distribution $\Gam(\nu/2,\nu/2)$, and given $\tau$, the random matrix $\bX$ follows the matrix-normal distribution $\cN_{c,r}(\bM,\bSig_c/\tau,\bSig_r)$, that is, 
\begin{equation}\label{eqn:mvt.hrc}
	\textsf{Model IV}: \quad \bX|\tau\sim\cN_{c,r}(\bM,\bSig_c/\tau,\bSig_r),\,\, \tau\sim\Gam(\nu/2,\nu/2).
\end{equation}
Under \textsf{Model IV} in \refe{eqn:mvt.hrc}, we can obtain the marginal distribution $\bX\sim Mt_{c,r}(\bM, \bSig_c, \bSig_r, \nu)$ or $\vc(\bX)\sim t_{cr}(\vc(\bM), \bSig_r\otimes\bSig_c,\nu)$, and hence
the covariance matrix 
\begin{equation}
	\cov{(\vc(\bX))}=\frac{\nu}{\nu-2}\bSig_r\otimes\bSig_c.\label{eqn:mvt.covx}
\end{equation} 
Moreover, we can obtain the posterior distribution 
\begin{equation}
	\tau|\bx\sim\Gam\left(\frac{\nu+cr}{2},\frac{\nu+\tr\{\bSig_c^{-1}(\bX_n-\bM)\bSig_r^{-1}(\bX_n-\bM)'\}}{2}\right).
	\label{eqn:mvt.tau|X}
\end{equation} 
By comparing \refe{eqn:mvt.hrc} with \refe{eqn:t.hrc1}, it can be seen that \textsf{Model IV} is an extension from \textsf{Model I} that replaces the multivariate normal distribution with a matrix-normal distribution. Therefore, the matrix-\emph{t} distribution is a natural extension of the multivariate \emph{t} distribution to matrix cases. Although \textsf{Model I} in \refe{eqn:t.hrc1} and \textsf{Model II} in \refe{eqn:t.hrc2} lead to the same multivariate \emph{t} distribution, their extensions to matrix cases, i.e., \textsf{Model III} in \refe{eqn:T.hrc} and \textsf{Model IV} in \refe{eqn:mvt.hrc}, yield two different types of matrix-variate distributions: the matrix-\emph{T} and matrix-\emph{t} distributions. 

\subsection{The proposed robust FPCA (RFPCA)}\label{sec:MtPCA.model}
Based on the matrix-\emph{t} distribution in \refs{sec:matrix-t}, we propose a robust FPCA, denoted by RFPCA. Under the matrix-\emph{t} distribution assumption, we have the factorization given in \refe{eqn:mvt.covx}. Similar to that of FPCA in \refs{sec:mn.fpca}, substituting the factorization $\bSig=\bSig_r\otimes\bSig_c$ (up to a constant) into \refe{eqn:fpca.obj} and \refe{eqn:fpca.U}, we can obtain the $\bU_c$ and $\bU_r$ solutions of RFPCA, where $\bSig_c$ and $\bSig_r$ can be estimated by $\btSig_c$ in \refe{eqn:ecme.Sc} and $\btSig_r$ in \refe{eqn:ecme.Sr}, respectively, as detailed in \refs{sec:matrix-t.mle}. 

\subsection{Maximum likelihood estimation of the matrix-\emph{t} distribution}\label{sec:matrix-t.mle}
In order to obtain the ML estimate of the parameter $\btheta=(\bM, \bSig_c, \bSig_r, \nu)$ in the matrix-\emph{t} distribution, we develop estimation algorithms in this section. Given a set of observations $\cX=\{\bX_n\}_{n=1}^N$, from \refe{eqn:mvt.density} the observed data log likelihood $\cL$ is, up to a constant, 
\begin{IEEEeqnarray}{rCl}\label{eqn:tfpca.like}
	\cL(\btheta|\cX)&=&-\frac12\sum\nolimits_{n=1}^N\left\{(\nu+cr)\ln{(\nu+\tr\{\bSig_c^{-1}(\bX_n-\bM)\bSig_r^{-1}(\bX_n-\bM)'\})}\right.\nonumber\\ &&\hskip-0.5em\left.+\>r\ln|\bSig_c|+c\ln|\bSig_r|\right\}+N\left\{\ln{\Gamma(\frac{\nu+cr}2)}-\ln{\Gamma(\frac{\nu}2)}+\frac{\nu}2\ln{\nu}\right\}.
\end{IEEEeqnarray}
To our knowledge, it is intractable to maximize $\cL$ in \refe{eqn:tfpca.like} directly and hence we shall make use of EM-type algorithms. As is well known, the speed of the EM algorithm depends on the proportion of missing information in the complete data and the traditional EM algorithm may converge very slowly \citep{meng-aecm}. Therefore, we propose two accelerated EM-type algorithms. The first in \refs{sec:RFPCA.ECME} is an expectation conditional maximization either (ECME) algorithm \citep{chuanhai_ecme}, which involves smaller amounts of missing information in some conditional maximization (CM)-steps. The second in \refs{sec:PX-ECME} is a variant of ECME that uses the parameter expansion technique \citep{liu-PXEM} for further acceleration (PX-ECME).

\subsubsection{An ECME algorithm}\label{sec:RFPCA.ECME}
The ECME algorithm consists of one expectation step (E-step) and four CM-steps. Each CM step maximizes a subset of parameters while keeping the other parameters fixed. We divide the parameters into $\btheta_1=\bM$, $\btheta_2=\bSig_c$, $\btheta_3=\bSig_r$ and $\btheta_4=\nu$. 

\noindent
{\bf E-step:} Let $\cT=\{\tau_n\}_{n=1}^N$ be the missing data. The complete data log likelihood function is $\sum\nolimits_{n=1}^N\ln{ \{p(\bX_n|\tau_n)p(\tau_n)\}}$, and its expectation w.r.t. the conditional distribution $p(\cT|\cX,\btheta)$ is, up to a constant, 
\begin{IEEEeqnarray}{rCl}
	Q(\btheta|\cX)&=&N\left\{\frac\nu2\ln{\frac\nu2}-\ln{\Gamma(\frac\nu2)}-\frac{r}2\ln|\bSig_c|-\frac{c}2\ln|\bSig_r|\right\}+\sum\nolimits_{n=1}^N\left\{\frac\nu2(\bbE[\ln{\tau_n|\bX_n}]\right.\nonumber\\
	&&\hskip0.5em\left.-\>\bbE[\tau_n|\bX_n])-\bbE[\tau_n|\bX_n]\tr\{\bSig_c^{-1}(\bX_n-\bM)\bSig_r^{-1}(\bX_n-\bM)'\}\right\}.\label{eqn:tfpca.Q}
\end{IEEEeqnarray}
By \refe{eqn:pdf.tau|x}, \refe{eqn:gam.etau} and \refe{eqn:gam.elntau}, the required expectations are computed by
\begin{IEEEeqnarray}{rCl}
	\bbE[\tau_n|\bX_n]&=& \frac{\nu+cr}{\nu+\tr\{\bSig_c^{-1}(\bX_n-\bM)\bSig_r^{-1}(\bX_n-\bM)'\}},\label{eqn:Etau.X}\\
	\bbE[\ln\tau_n|\bX_n]&=& \psi(\frac{\nu+cr}{2})-\ln(\frac{\nu+\tr\{\bSig_c^{-1}(\bX_n-\bM)\bSig_r^{-1}(\bX_n-\bM)'\}}{2}),\nonumber
\end{IEEEeqnarray}
where $\bbE[\tau_n|\bX_n]$is the expected weight of $\bX_n$, and $\psi(x)=d\ln(\Gamma(x))/dx$ is the digamma function. 

\noindent
{\bf CMQ-step 1:} Given $(\btheta_2,\btheta_3,\btheta_4)$, maximize $Q$ in \refe{eqn:tfpca.Q} w.r.t. $\btheta_1=\bM$ yielding
\begin{equation}\label{eqn:ecme.M}
	\btM=\frac1{\sum\nolimits_{n=1}^N\bbE[\tau_n|\bX_n]}\sum\nolimits_{n=1}^N\bbE[\tau_n|\bX_n]\bX_n. 
\end{equation}

\noindent
{\bf CMQ-step 2:} Given $(\bttheta_1,\btheta_3,\btheta_4)$, maximize $Q$ in \refe{eqn:tfpca.Q} w.r.t. $\btheta_2=\bSig_c$ leading to
\begin{equation}\label{eqn:ecme.Sc}
	\btSig_c = \frac{1}{Nr}\sum\nolimits_{n=1}^N\bbE[\tau_n|\bX_n] (\bX_n-\btM)\bSig_r^{-1}(\bX_n-\btM)'.
\end{equation}

\noindent
{\bf CMQ-step 3:} Given $(\bttheta_1,\bttheta_2,\btheta_4)$, maximize $Q$ in \refe{eqn:tfpca.Q} w.r.t. $\btheta_3=\bSig_r$ yielding
\begin{equation}\label{eqn:ecme.Sr}
	\btSig_r = \frac{1}{Nc}\sum\nolimits_{n=1}^N\bbE[\tau_n|\bX_n] (\bX_n-\btM)'\btSig_c^{-1}(\bX_n-\btM).
\end{equation}

\noindent
{\bf CML-step 4:} Given $(\bttheta_1,\bttheta_2,\bttheta_3)$, maximize the log likelihood $\cL$ in \refe{eqn:tfpca.like} w.r.t. $\btheta_4=\nu$. Set the derivative of $\cL$ w.r.t. $\nu$ equal to 0. The updated $\tilde{\nu}$ can be obtained by solving the quation
\begin{equation}\label{eqn:ecme.nu}
	-\psi(\frac{\nu}2)+\ln{(\frac{\nu}2)}+1+\psi(\frac{\nu+cr}2)-\ln{(\frac{\nu+cr}2)}+\frac1N\sum\nolimits_{n=1}^N\left[\ln{\left(\frac{\nu+cr}{\nu+\tilde{\delta}_{\bX_n}}\right)}-\frac{\nu+cr}{\nu+\tilde{\delta}_{\bX_n}}\right]=0,
\end{equation}
where $\tilde{\delta}_{\bX_n}=\tr\{\btSig_c^{-1}(\bX_n-\btM)\btSig_r^{-1}(\bX_n-\btM)'\}$, and $\psi(x)=d\ln(\Gamma(x))/dx$ is the digamma function.

The complete ECME algorithm is summarized in \refa{alg:ecme}. It can be seen that all CM-steps are guaranteed to increase the log likelihood at each iteration. Under the same mild conditions as the standard EM, the ECME is guaranteed to converge to a stationary point of $\cL$ \citep{chuanhai_ecme}. 
\begin{algorithm}[htb]
	\caption{The ECME (resp. PX-ECME) algorithm for RFPCA.}
	\label{alg:ecme}
	\begin{algorithmic}[1]
		\REQUIRE Data $\cX$ and (random) initialization of $\btheta=(\bM,\bSig_c,\bSig_r,\nu)$.
		\REPEAT
		\STATE \emph{E-step:} Compute the conditional expectation $\bbE[\tau_n|\bX_n]$ via \refe{eqn:Etau.X}.
		\STATE \emph{CM-step 1:} Update $\btM$ via \refe{eqn:ecme.M}.
		\STATE \emph{CM-step 2:} Update $\btSig_c$ via \refe{eqn:ecme.Sc} (resp. \refe{eqn:px-ecme.Sc}). 
		\STATE \emph{CM-step 3:} Update $\btSig_r$ via \refe{eqn:ecme.Sr} (resp. \refe{eqn:px-ecme.Sr}). 
		\STATE \emph{CM-step 4:} Set $\tilde{\nu}$ to be the solution of equation \refe{eqn:ecme.nu}.		
		\UNTIL{change of $\cL$ is smaller than a threshold.}
		\ENSURE $\bttheta=(\btM,\btSig_c,\btSig_r,\tilde{\nu})$.
	\end{algorithmic}
\end{algorithm}

Next, we analyze the computational complexity of \refa{alg:ecme}. The main computations lie in CM-step 2 and 3. In CM-step 2, the cost of computing $\bSig_r^{-1}$ is $O(r^3)$. Given $\bSig_r^{-1}$, computing $\bbE[\tau_n|\bX_n](\bX_n-\btM)\bSig_r^{-1}(\bX_n-\btM)'$ takes $O(cr^2+c^2r)$ time. Hence, the cost of computing $\btSig_c$ is $O(N[cr^2+c^2r])+O(r^3)$. Similarly, in CM-step 3, computing $\btSig_r$ costs $O(N[c^2r+cr^2])+O(c^3)$. Summing over all number of iterations $t$ gives the total cost $O(tN[c^2r+cr^2])+O(t[c^3+r^3])$. For clarity, \reft{tab:complex2} summarizes the main computational complexity of FPCA, TPCA and RFPCA. It can be seen that (i) RFPCA and FPCA have similar per-iteration complexity. (ii) RFPCA has lower per-iteration cost than TPCA, becuase TPCA involves the computations of the inverses and determinants of $N$ matrices of size $c\times c$, which is $O(Nc^3)$. Due to this additional cost only depending on $c$, \cite{thompson2020classification} suggest the orientation of the matrix observations $\bX_n$ in TPCA should be chosen so that $c<r$.

\begin{table}[htbp]
	\centering
	\caption{\label{tab:complex2} Computational complexity by different methods.}
	\begin{tabular}{cc}
		\toprule 
		Method        &   Computational cost   \\ \midrule	
		FPCA          &  $O(tN[c^2r+cr^2])+O(t[c^3+r^3])$\\ 
		TPCA          &  $O(tN[c^2r+cr^2+c^3])+O(t[c^3+r^3])$\\
		RFPCA  &  $O(tN[c^2r+cr^2])+O(t[c^3+r^3])$\\ \bottomrule
	\end{tabular}
\end{table}

\subsubsection{Accelerated ECME via the parameter expansion technique (PX-ECME)}\label{sec:PX-ECME}
To improve the convergence speed of traditional EM-type algorithms, \cite{liu-PXEM} propose a parameter expansion technique. This technique introduces auxiliary parameters to expand the original model while keeping the observed data model unchanged. The expanded complete data model is then used to develop EM-type algorithms. The resulting algorithm often yields faster convergence in that it can capitalize on the extra information captured by the auxiliary parameters for the complete data model, as demonstrated by \cite{liu-PXEM} in a number of applications including the multivariate $t$ distribution. In this subsection, we explore the application of this technique to the matrix-\emph{t} distribution, and develop a parameter-expanded ECME (PX-ECME) algorithm. 

Introduce an auxiliary parameter $\alpha$ and expand \textsf{Model IV} of the matrix-\emph{t} distribution \refe{eqn:mvt.hrc} as
\begin{equation}\label{eqn:mvt.hrc.PX}
	\textsf{Model V}: \quad \bX|\tau\sim\cN_{c,r}(\bM_*,\bSig_{c*}/\tau,\bSig_{r*}),\,\, \tau\sim\Gam(\nu_*/2,\nu_*/{2\alpha}).
\end{equation}
where the parameters are marked with the subscript ${}_*$ to distinguish \textsf{Model V} \refe{eqn:mvt.hrc.PX} from original \textsf{Model IV} \refe{eqn:mvt.hrc}. Under \textsf{Model V} \refe{eqn:mvt.hrc.PX}, we obtain $\bX\sim Mt_{c,r}(\bM_*, \bSig_{c*}/\alpha, \bSig_{r*}, \nu_*)$. It can be seen that $\bSig_{c*}/\alpha$ corresponds to $\bSig_c$ and $(\bM_*, \bSig_{r*}, \nu_*)$ to $(\bM,\bSig_r,\nu)$. By using this mapping, the expanded parameters can be reduced back to the original parameters at the end of each iteration. 
In this way, as detailed in \refap{sec:PX-ECME.ECME}, it can be shown that the PX-ECME can be represented as a modified ECME except that $\btSig_c$ in \refe{eqn:ecme.Sc} and $\btSig_r$ in \refe{eqn:ecme.Sr} are replaced by \refe{eqn:px-ecme.Sc} and \refe{eqn:px-ecme.Sr}, respectively,
\begin{IEEEeqnarray}{rCl}
	\btSig_c& = & \frac{\sum\nolimits_{n=1}^N \bbE[\tau_n|\bX_n] (\bX_n-\btM)\bSig_r^{-1}(\bX_n-\btM)'}{r\sum\nolimits_{n=1}^N\bbE[\tau_n|\bX_n]},\label{eqn:px-ecme.Sc}\\
	\btSig_r& = & \frac{\sum\nolimits_{n=1}^N \bbE[\tau_n|\bX_n] (\bX_n-\btM)'\btSig_c^{-1}(\bX_n-\btM)}{c\sum\nolimits_{n=1}^N\bbE[\tau_n|\bX_n]}.\label{eqn:px-ecme.Sr}
\end{IEEEeqnarray}
For clarity, the complete PX-ECME algorithm as a modified ECME algorithm is also summarized in \refa{alg:ecme}, where the differences from the ECME are indicated by resp. It can be seen from \refa{alg:ecme} that both PX-ECME and ECME have the same computational complexity.

\subsubsection{Outlier detection by RFPCA}\label{sec:mvt.odtc}
Since the matrix-\emph{t} distribution is essentially a multivariate \emph{t} distribution with a separable covariance matrix $\bSig=\bSig_r\otimes \bSig_c$, the matrix-\emph{t} distribution naturally inherits the important characteristics of the multivariate \emph{t} distribution mentioned in \refs{sec:t.mle}. Specifically,

(i) The degrees of freedom $\nu$ is able to tune the robustness. For example, when data are distributed as a matrix-variate normal distribution, it is expected the ML estimate of $\nu$ would take large values, which makes all the expected weights $\bbE[\tau_n|\bX_n]$ in \refe{eqn:Etau.X} close to 1. As a result, the matrix-\emph{t} distribution tends to be a matrix-normal distribution in this case. When the distribution of data deviates from a matrix-normal distribution, $\nu$ may take small values. 

(ii) The expected weight $\bbE[\tau_n|\bX_n]$ in \refe{eqn:Etau.X} can be used to detect outliers. 
\begin{prop}\label{prop:mvt.weight}
	Suppose that the data $\{\bX_n\}_{n=1}^N$ follow the matrix-\emph{t} distribution $Mt_{c,r}(\bM, \bSig_c, \bSig_r, \nu)$. For the ML estimate $\btheta$, the average expected weight for all $N$ matrix observations equals to 1, that is,
	\begin{equation}
		\frac1N\sum\nolimits_{n=1}^N\bbE[\tau_n|\bX_n]=1.\label{eqn:mvt.ave.tau}
	\end{equation} 
\end{prop}

\begin{proof}
	The proof can be found in \refap{sec:Prop1.proof}.
\end{proof}

According to \refp{prop:mvt.weight}, the observation with the weight $\bbE[\tau_n|\bx_n]$ much smaller than 1 or close to 0 can be judged as an outlier. When data contain outliers, it is expected that the degrees of freedom $\nu$ would take small values, and the Mahalanobis distances of outliers $\tr\{\bSig_c^{-1}(\bX_n-\bM)\bSig_r^{-1}(\bX_n-\bM)'\}$ would be large, which cause the outliers' weights much smaller than those of non-outliers.
Obviously, both \emph{t}PCA and RFPCA can be used for outlier detection. However, their reliability may be different. We empirically compare them detailedly in \refs{sec:expr}. 

\section{Experiments}\label{sec:expr}
In this section, we conduct experiments on synthetic and real-world datasets to compare the performance of our proposed RFPCA with five related methods, including matrix-\emph{T} PCA (TPCA) \citep{thompson2020classification}, factored PCA (FPCA) \citep{dryden-sbpca} and multivariate-\emph{t} PCA (\emph{t}PCA), bidirectional PCA (BPCA) \citep{zhangdq-2dpca} and PCA. Unless otherwise stated, random initialization is used for all iterative algorithms and the iterations are stopped when the relative change of the objective function ($|1-\cL^{(t)}/\cL^{(t+1)}|$) is smaller than a threshold $tol$ ($=10^{-8}$ in our experiments) or the number of iterations exceeds a certain value $t_{max}$($=1000$). All computations are implemented by Matlab 9.9 on a desktop computer with Intel Core i7-5820K 3.30GHz CPU and 32G RAM.

\subsection{Synthetic data}
\subsubsection{Convergence of ECME and PX-ECME}\label{sec:cov}
In this experiment, we investigate the convergence performance of the proposed ECME and PX-ECME for RFPCA using two synthetic datasets from matrix-$t$ distributions. The first Data1 is a low-dimensional dataset and the second Data2 is a high-dimensional one. Let $\mbox{linspace}(a,b,n)$ represent a vector of $n$ points equally spaced between $a$ and $b$, and $\bu_1=[\frac{1}{\sqrt{2}}, -\frac{1}{\sqrt{2}}, 0,$ $ \dots,0]'$, $\bu_2=[ 0, 0, \frac{1}{\sqrt{2}}, -\frac{1}{\sqrt{2}},$ $ 0, \dots,0]'$, $\bu_3=[0, 0, 0, 0,\frac{1}{\sqrt{2}}, -\frac{1}{\sqrt{2}}, 0, \dots,0]'$.

(i) Low-dimensional Data1: Data1 is generated from $Mt_{4,10}(\bM, \bSig_c, \bSig_r, \nu)$, where $\bM=\bo$, $\nu=3$, the eigenvalues of $\bSig_c$ and $\bSig_r$ are $(5,\mbox{linspace}(0.8,0.5,3))$ and $(4,3,2,\mbox{linspace}(0.5,0.3,7))$ respectively. Since the first eigenvalue of $\bSig_c$ is dominant, the number of principal components $q_c=1$. Similarly, $q_r=3$. The leading principal components are set as $\bU_c=\bu_1$ and $\bU_r=(\bu_1,\bu_2,\bu_3)$, respectively. The sample size is 500. 

(ii) High-dimensional Data2: Data2 is generated from $Mt_{100,100}(\bM, \bSig_c, \bSig_r, \nu)$, and its parameters are the same as Data1 except for the last 97 eigenvalues of $\bSig_r$ and $\bSig_c$ are $\mbox{linspace}(0.5,0.3,97)$ and $\mbox{linspace}(0.8,$ $0.5,97)$ respectively. The sample size is still 500.
\begin{figure*}[htb]
	\centering \scalebox{0.7}[0.75]{\includegraphics*{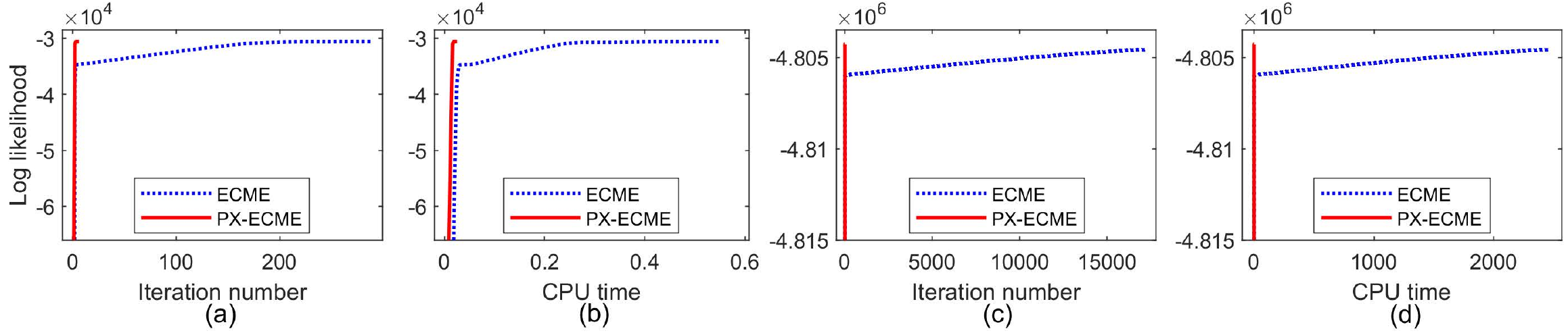}}
	\caption{Typical evolvements of data log-likelihood for the ECME (dotted) and PX-ECME (solid) algorithms versus (a) number of iterations on Data1, (b) CPU time on Data1, (c) number of iterations on Data2, and (d) CPU time on Data2.} \label{fig:cov}
\end{figure*}

We fit the two datasets using the ECME and PX-ECME, respectively. For fair comparison, we use the same initialization for both algorithms. For demonstration purpose, we set $t_{max}=20000$ when fitting Data2, since the ECME algorithm converges slowly in this high-dimensional case. \reff{fig:cov} (a) and (b) respectively show the evolution of log likelihood versus the number of iterations and used CPU time on Data1; \reff{fig:cov} (c) and (d) show the results on Data2. It can be seen that PX-ECME converges significantly faster than ECME no matter whether in low-dimensional or high-dimensional cases. Therefore, we shall use the more efficient PX-ECME algorithm for RFPCA in our experiments. 

\subsubsection{Accuracies of estimators}\label{sec:acc}
In this experiment, we compare the accuracies of the estimators between RFPCA and \emph{t}PCA in finite sample cases. For comparison, the performances of TPCA and FPCA are also included. We sample data from the Data1 in \refs{sec:cov} but with two values of $\nu$: $\nu=3$ and $\nu=30$, where the added case of $\nu=30$ is close to a matrix-normal distribution.

We generate training data with varying sample size $N$ from the set $\{50, 100, 200, 500, 1000\}$ to investigate the finite sample performance. To compare the estimates obtained by different methods, we choose the test data log-likelihood as performance measure. An independent test dataset of size $N=5000$ is generated to calculate the test data log-likelihood. In order to reduce variability, we report the average results over 20 repetitions. 

\begin{figure*}[htb]
	\centering
	\scalebox{0.75}[0.75]{\includegraphics{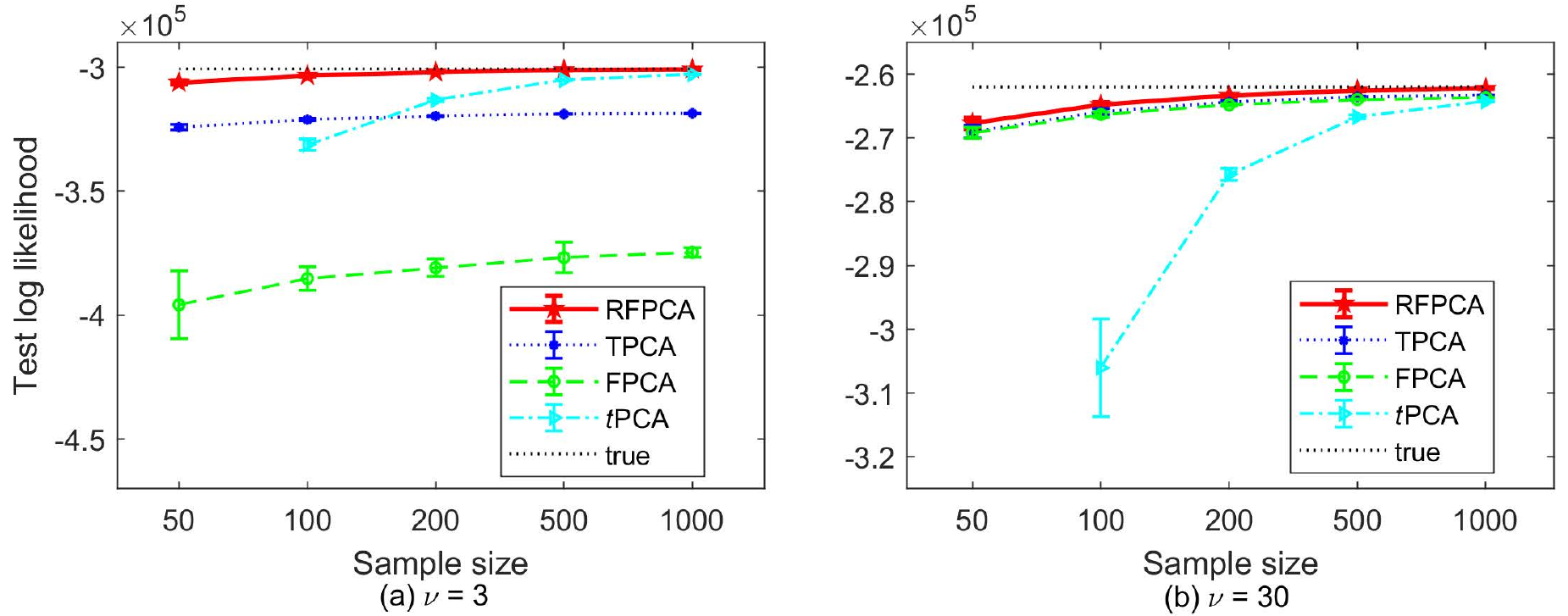}}
	\caption{The average test log-likelihood and standard deviation versus training sample size $N$ for (a) $\nu=3$ and (b) $\nu=30$.}
	\label{fig:llh}
\end{figure*}

\reff{fig:llh} (a) and (b) show the curve of test data log-likelihood by different methods versus the training sample size $N$ for $\nu=3$ and $\nu=30$, respectively. The following can be observed from \reff{fig:llh} (a) and (b) 

(i) No matter whether $\nu=3$ or $\nu=30$, as the training sample size increases, the test likelihoods of both RFPCA and \emph{t}PCA converge to the true one. Nevertheless, RFPCA is substantially better than \emph{t}PCA in finite sample cases as expected.

(ii) When $\nu=30$, as the training sample size increases, the test likelihoods of all methods are close to the true one. When $\nu=3$, FPCA and TPCA yield the worse performance. This observation is not surprising as the data is generated from the matrix-\emph{t} distribution. Nevertheless, this observation also indicates that RFPCA has wider applicability than FPCA since FPCA performs satisfactorily only for large $\nu$ cases. 

\subsubsection{Robustness}\label{sec:robust}
In this experiment, we examine the robustness of all six methods using a synthetic dataset containing abnormal observations. We first generate a dataset with the sample size $N=1000$ from the $4 \times 10$-dimensional matrix-normal distribution $\cN(\bM,\bSig_c,\bSig_r)$, where the parameters $\bM$, $\bSig_c$ and $\bSig_r$ are the same as those of Data1 in \refs{sec:cov}. After the dataset is generated, we add it with $Np$ outliers from the uniform distribution $U(100,110)$. This data is denoted as Data3 for clarity.

To compare the estimates obtained by different methods, we report the F-norm of the difference from the true covariance matrix $\bSig=\bSig_r\otimes\bSig_c$ to the estimated $\bwSig$, namely $\Vert \bSig-\bwSig \Vert_F$. For PCA, $\bwSig$ is the sample covariance matrix. For \emph{t}PCA, $\bwSig$ is the robust sample covariance matrix \refe{eqn:ecme.t.S}. For BPCA, $\bwSig=\bS_r\otimes\bS_c/\tr(\bS_c)$ \citep{zhao2012-slda}, where $\bS_c=\frac{1}{N}\sum_{n=1}^{N}(\bX_n-\bar{\bX})(\bX_n-\bar{\bX})'$, $\bS_r=\frac{1}{N}\sum_{n=1}^{N}(\bX_n-\bar{\bX})'(\bX_n-\bar{\bX})$ ($\bar{\bX}$ is the sample mean). For FPCA and RFPCA, $\bwSig=\bwSig_r\otimes\bwSig_c$, where ($\bwSig_c$, $\bwSig_r$) is obtained by \refe{eqn:fpca.Sc}, \refe{eqn:fpca.Sr} and \refe{eqn:px-ecme.Sc}, \refe{eqn:px-ecme.Sr}, respectively. For TPCA, by comparing \refe{eqn:mvt.covx} with \refe{eqn:T.covx}, we use $\bwSig=\bwSig_r\otimes\bwSig_c/\hat{\nu}$, where $\bwSig_c$ and $\bwSig_r$ are obtained by \refe{eqn:T.Sc} and \refe{eqn:T.Sr}. To reduce statistical variability, we report the average results over 50 repetitions. 

\begin{table*}[htb]
	\centering
	\caption{\label{tab:robust} Euclidean distance between the estimated and true covariance matrix for multiple $p$'s.}
	\begin{tabular}{ccccccc}
		\toprule
		\multirow{2}[2]{*}{Proportion $p$} & \multicolumn{6}{c}{Method} \\
		\cmidrule{2-7}        & RFPCA &    TPCA    &   FPCA   &   BPCA  & \emph{t}PCA & PCA   \\ \midrule
		0                     &      1.1     &    1.1     &   1.1    &   1.6   &     2.7     & 2.7 \\
		2\%                  &      2.1     &    1.5     &  7968.1  &  8404.2 &     3.6     & 8477.5 \\
		3\%                  &      2.5     &    2.5     & 10645.5  & 12397.9 &    230.1    & 12472.9 \\ 
		7\%                  &      5.4     &    6.7     & 19282.9  & 26875.4 &   1229.7    & 26956.6 \\ 
		9\%                   & 12.8  &    9.0     & 22755.3  & 33322.2 &   1735.7    &33406.0 \\ \bottomrule		
	\end{tabular}
\end{table*}

\reft{tab:robust} summarizes the results for various proportions $p$. The main observations include

(i) RFPCA versus FPCA, BPCA and PCA: RFPCA is robust while FPCA, BPCA and PCA are all vulnerable to outliers. The larger the value of $p$, the greater the distance.

(ii) RFPCA versus \emph{t}PCA: \emph{t}PCA is robust when $p\leq2\%$, but it is not robust when $p\geq3\%$. In contrast, RFPCA is robust for all values of $p$. This means that RFPCA can accommodate a larger proportion of outliers compared with \emph{t}PCA.

(iii) RFPCA versus TPCA: Both RFPCA and TPCA are robust. Although the original intention of TPCA proposed in \cite{thompson2020classification} is not to propose a robust method, but to improve the fitting for matrix data, we find in this paper that TPCA is also a robust method. 

The robustness of estimators can be measured by breakdown points. As shown in \cite{Dumbgen-bkd}, the breakdown point of the multivariate \emph{t} distribution is upper bounded by $1/(d+\nu)(<1/d)$. For the data we used, $d=cr=40$ and hence $1/d=0.025$. This is the reason why \emph{t}PCA performs well when $p\leq2\%$ but fails completely when $p\geq3\%$. However, this theoretical bound of the multivariate \emph{t} distribution does not hold true for RFPCA, as RFPCA still performs well even when $p=9\%$. This reveals that the proposed RFPCA has a significantly higher bound than \emph{t}PCA, which makes RFPCA more applicable to matrix data.

\subsubsection{Outlier detection}\label{sec:simu.odtc}
As detailed in \refs{sec:mvt.odtc}, both \emph{t}PCA and RFPCA can be used for outlier detection. In this experiment, we compare their performance using synthetic datasets with different abnormal levels. The first dataset is the Data3 used in \refs{sec:robust}, namely the sample size is $N(1+p)$, where $N=1000$ for normal observations and $Np$ for abnormal ones from $U(100,110)$. The second Data3-2 and third Data3-3 are the same as the first except that the abnormal observations are drawn from $U(100,102)$ and $U(100000,100002)$, respectively. We set $p=0.05$ for all the three datasets.

As a robust method, it is expected that the weights of outliers should be as small as possible, so that the influence of the outliers to the estimates can be down-weighted. This can be seen from \refe{eqn:ecme.t.mu} and \refe{eqn:ecme.t.S} for \emph{t}PCA, and from \refe{eqn:ecme.M}, \refe{eqn:ecme.Sc} and \refe{eqn:ecme.Sr} for RFPCA.

\reff{fig:simu.oidc} shows the scatter plot of the weights for all normal and outlying observations on the three datasets. It can be seen that (i) on Data3, both RFPCA and \emph{t}PCA are competent for this task as the weights of the outliers are much smaller than those of the normal observations; (ii) However, for Data3-2 and Data3-3, RFPCA performs satisfactorily while \emph{t}PCA fails completely as the normal observations and outliers have similar weights. Therefore, RFPCA can be much more reliable than \emph{t}PCA for outlier detection; In addition, the stable performance by RFPCA from Data3-2 to Data3-3 reveals that RFPCA can be used for detecting not only mild but also gross outliers, while the freshly proposed method in \cite{tomarchio2021mixtures} can only be used to detect mild outliers. 
\begin{figure*}[htb]
	\centering \scalebox{0.7}[0.65]{\includegraphics*{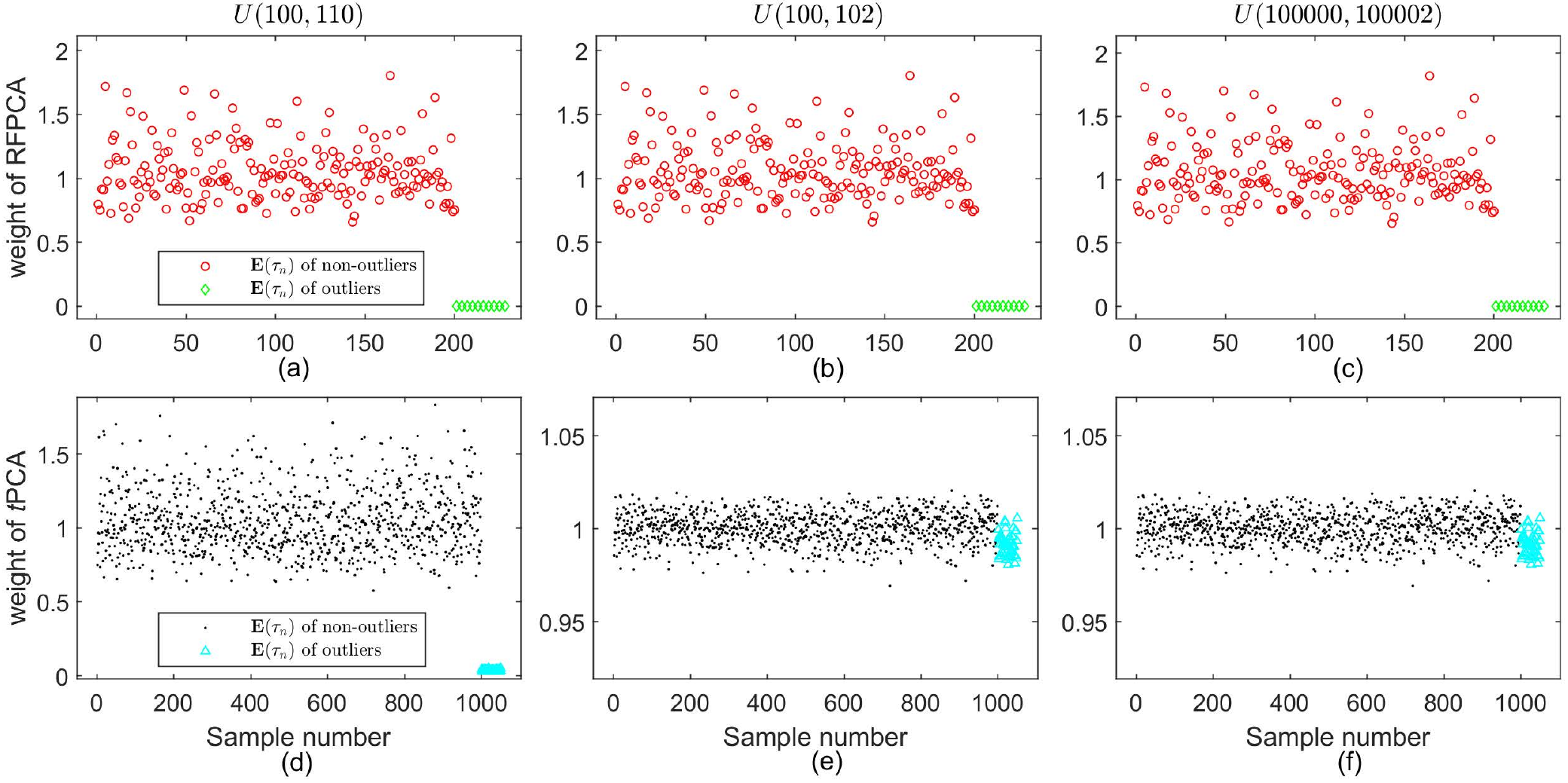}}
\caption{Top row: Weights of non-outliers (circle) and outliers (diamond) by RFPCA. Bottom row: Weights of non-outliers (point) and outliers (up triangle) by \emph{t}PCA.}\label{fig:simu.oidc}
\end{figure*} 

\subsubsection{Performance on computational efficiency}\label{sec:eff}
\reft{tab:complex2} gives the theoretical time complexity for RFPCA, TPCA and FPCA. In this experiment, we empirically compare their computational efficiency using a synthetic dataset contaminated by outliers. We generate data from the $100\times100$-dimensional matrix-normal distribution $\cN_{100,100}(\bM,\bSig_c,\bSig_r)$, where the parameters $\bM$, $\bSig_c$ and $\bSig_r$ are set as those of Data2 in \refs{sec:cov}. The sample size is selected from the set $\{200,500,1000,2000,4000,$ $8000,13000\}$. After the dataset is generated, we add it with $Np$ outliers from the uniform distribution $U(100,110)$, where $p=0.5\%$. 

\begin{figure*}[htb]
	\centering \scalebox{0.7}[0.65]{\includegraphics*{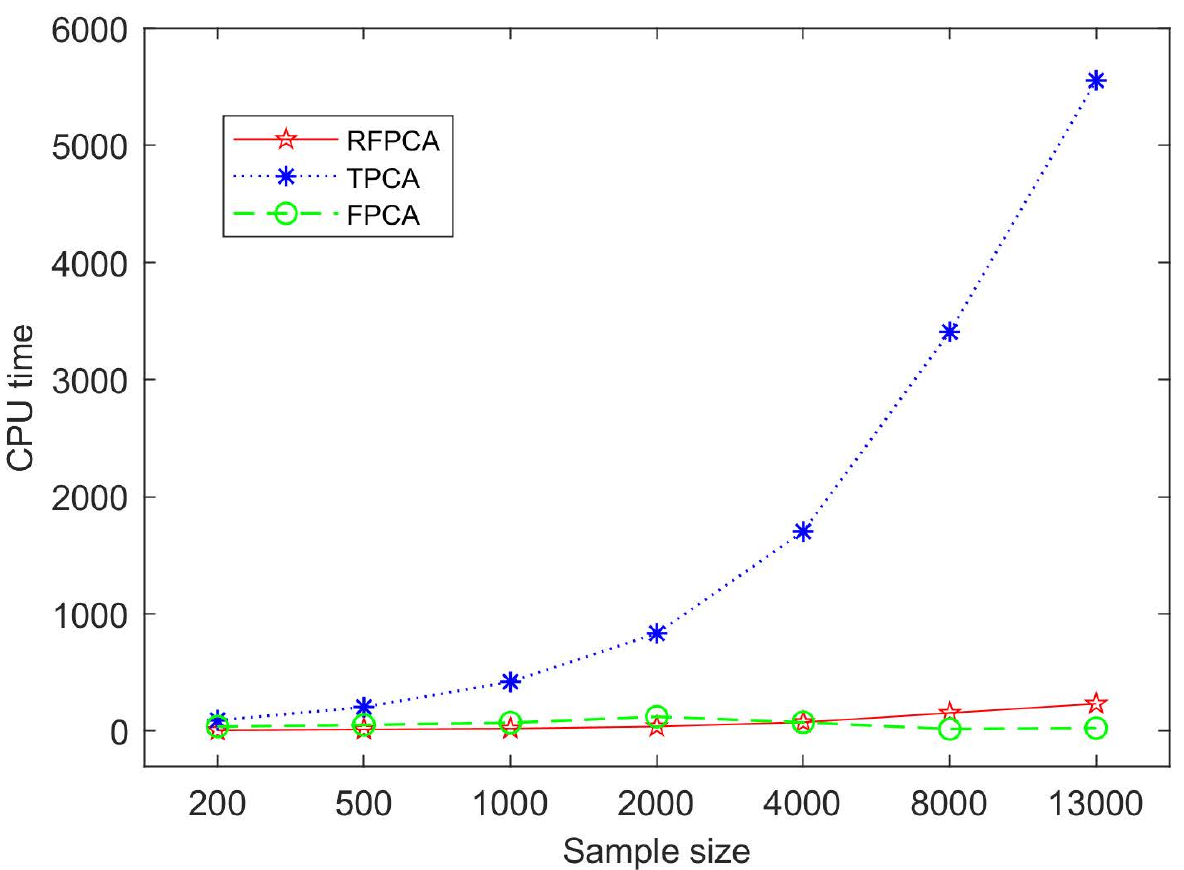}} 
	\caption{CPU time in seconds used by different methods as the sample size varies.} \label{fig:time}
\end{figure*}
\begin{table}[htb]
	\centering
	\caption{\label{tab:eff} The used CPU time and numbers of iterations by different methods on synthetic datasets with varying sample sizes.}
	\begin{tabular}{ccccccc}
		\toprule
		\multirow{2}{*}{Sample size} & \multicolumn{3}{c}{CPU time}    & \multicolumn{3}{c}{Number of iterations}  \\ \cmidrule(r){2-4} \cmidrule(r){5-7}   
		&  RFPCA &   TPCA  &  FPCA  & RFPCA &  TPCA  &  FPCA  \\ \midrule
		500	      &  14.2  &  203.6  &  52.8  &  22   &   42   &    175  \\  
		2000	  &  39.1  &  834.3  & 122.5  &  18   &   43   &    125 \\ 
		8000	  & 153.8  &  3407.8 &  17.2  &  18   &   43   &    4 \\
		13000	  & 232.1  &  5553.6 &  25.3  &  18   &   43   &    4 \\ \bottomrule
	\end{tabular}
\end{table}
We fit the data using the three methods. \reff{fig:time} plots their used CPU time in seconds as the sample size varies. \reft{tab:eff} collects some detailed results for several sample sizes.

It can be seen from \reff{fig:time} that both RFPCA and FPCA are computationally much more efficient than TPCA, especially when the sample size gets large. At first glance, this observation might not be consistent with the theoretical analysis given in \reft{tab:complex2}, since TPCA only has slightly higher per-iteration complexity than RFPCA. Two reasons are given below. First, from \reft{tab:eff} RFPCA requires smaller number of iterations than TPCA, as RFPCA is fitted by our proposed fast PX-ECME algorithm. Second, at each iteration, TPCA involves the inverses and determinants of $N$ matrices of size $c\times c$ while RFPCA mainly includes matrix multiplication. Despite their similar complexity (as $c=r$ for this data), the matrix multiplication in practice can be much more efficient than the inverses and determinants on this high-dimensional dataset. 

In addition, as \emph{t}PCA can not be run for the cases $N<d=cr=10000$, we run \emph{t}PCA for the case $N=13000$. As a result, \emph{t}PCA requires 32 iterations and takes 19254.2 seconds, which is far more than the three methods.

\subsection{Face datasets}\label{sec:cls.face}
In this subsection, we perform experiments on the following two real-world face datasets.

\textbullet\, XM2VTS\footnote{Available from \url{http://www.face-rec.org/databases/}}, comprises four image acquisitions of 295 people in a period of four months. Each acquisition contains two images and hence each person has 8 images. The size of each image is $51\times55$.

\textbullet\, PIE face database\footnote{Available from \url{http://www.ri.cmu.edu/projects/project_418.html}} contains 41368 face images of 68 people. The facial images for each person were captured under 13 different poses, 43 different illumination conditions, and with four different expressions. In our experiments, we use the same sub-database as in \cite{zhao2015-rlda-2stage}, where each person has 43 images, composed of frontal images under different lighting conditions. The size of each image is $64\times64$.

Each dataset is divided into a training set and a test set. $r$ images from each person are randomly selected to form the training set of size $N_{tr}$, and the remaining $N_{ts}$ images are used as the test set. Two values of $\gamma$ are considered from the candidate set $\{3, 4,5,6\}$. To compare the robustness of different methods, we add the training set with $N_{tr}p$ outlying images. Their pixel values are randomly set to 0 or 1 with probability 0.5. We set $p=0.1$ in the following experiments. 

\subsubsection{Outlier detection}\label{sec:face.odtc}
In this experiment, we use the two face datasets to investigate the performance of detecting outlying images by RFPCA and \emph{t}PCA. However, the result by \emph{t}PCA is not available as \emph{t}PCA can not be run when the sample size $N_{tr}$ is smaller than the vectorized data dimension $d$. For example, $N_{tr}=1298, d=2805$ for XM2VTS ($\gamma=4$) and $N_{tr}=449, d=4096$ for PIE ($\gamma=6$). \reff{fig:oidc} shows the scatter plot of the weights for all the normal and outlying images on XM2VTS ($\gamma=4$) and PIE ($\gamma=6$). It can be seen that RFPCA works well on these two contaminated datasets.

\begin{figure*}[htb]
	\centering \scalebox{0.9}[0.9]{\includegraphics*{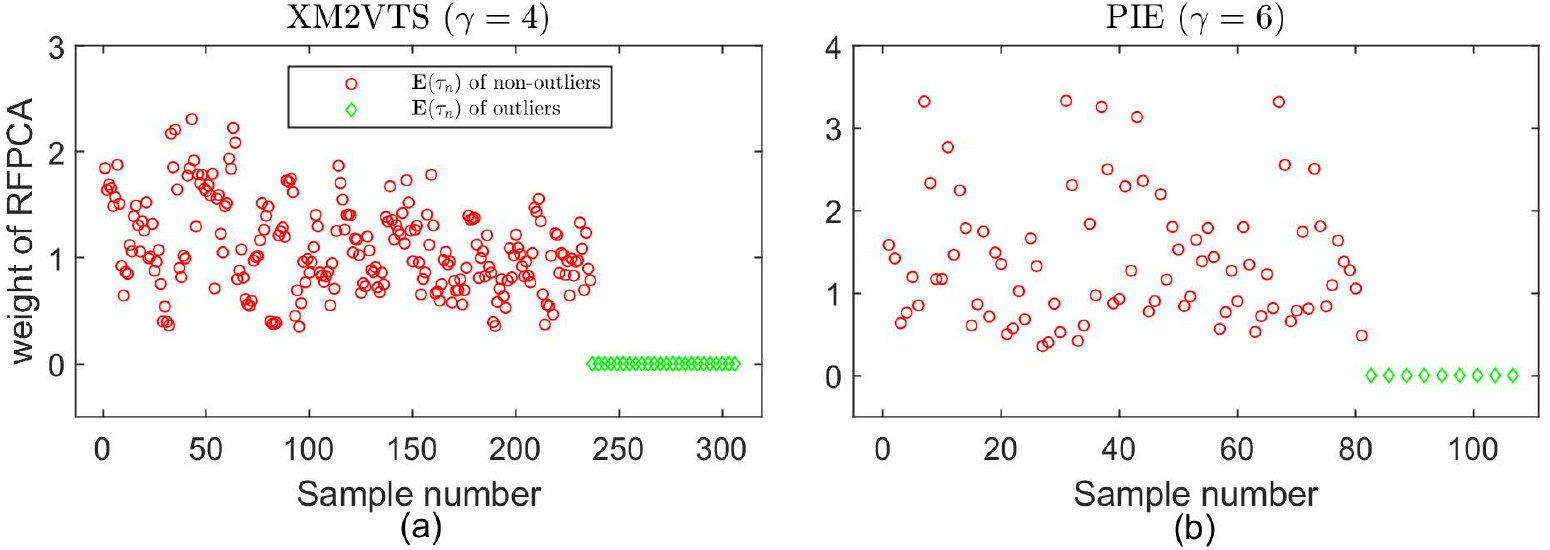}}
	\caption{Weights of non-outliers (circle) and outliers (diamond) by RFPCA on (a) XM2VTS and (b) PIE.} \label{fig:oidc}
\end{figure*}

\subsubsection{Classification performance on face datasets}\label{sec:cls.face.res}
In this experiment, we compare the classification performance of RFPCA with the other five methods. For PCA and \emph{t}PCA, the reduced representation in the $q$-dimensional space is $\bz=\bLmd^{-1/2}\bU'\bx$, where $(\bU,\bLmd)$ is detailed in \refs{sec:t.tPCA}.
For RFPCA, TPCA and FPCA, the reduced representation in the $(q_c,q_r)$-dimensional space is $\bZ=\bLmd_c^{-1/2}\bU_c'\bX\bU_r\bLmd_r^{-1/2}$, where $(\bU_c,\bLmd_c)$ and $(\bU_r,\bLmd_r)$ are detailed in \refs{sec:mn.fpca}. For BPCA, we follow \cite{zhangdq-2dpca} to use the reduced representation $\bZ=\bU_c'\bX\bU_r$. 

For all methods, we try all possible dimensionalities (i.e., all possible values of $q$ or $(q_c,q_r)$) of the reduced representation and use 1-nearest neighbor classifier in the reduced-dimensional space to obtain the classification error rates. To test whether the performance difference between two methods is significant, we use the Wilcoxon signed-rank test. 

\begin{table*}[htbp]
	\centering
	\caption{\label{tab:cls.face} The lowest average error rates (mean $\pm$ std) and their corresponding dimensions by different methods on the face datasets. The best method is shown in boldface. $\bullet$ means that RFPCA is significantly better than the method, and $\circ$ means that the difference between RFPCA and the method is not significant, using the Wilcoxon signed-rank test with a $p$-value of 0.05.}
	\resizebox{\linewidth}{!}{
		\begin{tabular}{ccccc}
			\toprule
			\multirow{2}{*}{Method} &   \multicolumn{2}{c}{XM2VTS ($\gamma=4$)}     & \multicolumn{2}{c}{PIE ($\gamma=6$)}         \\ \cmidrule(r){2-3} \cmidrule(r){4-5} 
			&        without outliers     &     with 10\% outliers       &        without outliers       &    with 10\% outliers  \\ \midrule
			RFPCA   & {\bf11.4}$\pm$2.0(17, 7)    & {\bf11.3}$\pm$2.3(12, 7)     &9.4$\pm$5.5(26, 12)            &{\bf9.3}$\pm$5.5(26, 12)\\ 
			TPCA    & 11.6$\pm$2.0(12, 7)$\circ$  & 11.6$\pm$2.2(12, 7)$\bullet$ &{\bf9.3}$\pm$5.1(26, 15)$\circ$&{\bf9.3}$\pm$5.1(27, 15)$\circ$\\
			FPCA    &11.6$\pm$2.1(12, 8)$\circ$   & 12.8$\pm$1.8(21, 10)$\bullet$&9.6$\pm$5.6(26, 14)$\bullet$   &10.7$\pm$6.5(64, 19)$\bullet$\\
			BPCA    &11.6$\pm$2.1(12, 8)$\bullet$ & 16.5$\pm$2.1(37, 51)$\bullet$&30.2$\pm$8.5(64, 64)$\bullet$  &30.2$\pm$8.5(64, 64)$\bullet$ \\
			\emph{t}PCA&               —             &               —              &             —                 &     —   \\
			PCA     &14.6$\pm$2.4(91)$\bullet$    & 14.7$\pm$2.1(199)$\bullet$   &10.9$\pm$7.2(74)$\circ$        &10.9$\pm$7.2(109)$\circ$\\ \bottomrule
		\end{tabular}
	}
\end{table*} 

\begin{table}[htb]
	\centering
	\caption{\label{tbl:time} Average training time in seconds by different methods.}
	\begin{tabular}{cccccc}
		\toprule 
		\multirow{2}{*}{Dataset} & \multicolumn{5}{c}{Method}\\ \cmidrule{2-6} 
		& RFPCA &   TPCA   & FPCA & BPCA & PCA  \\ \midrule
		XM2VTS ($\gamma$=4) &     4.84     &  95.88   & 1.43 & 0.12 & 2.06 \\
		PIE ($\gamma$=6)    &     3.03     &  42.00   & 1.10 & 0.10 & 0.40 \\ \bottomrule	  
	\end{tabular}
\end{table}
\reft{tab:cls.face} reports the results by all the methods over 10 random splitting of XM2VTS ($\gamma=4$) and PIE ($\gamma=6$), including the lowest average error rate, corresponding latent dimension and standard deviation, where the `—' sign means that the method fails to run and hence the error rate is not available. The results on XM2VTS ($\gamma=3$) and PIE ($\gamma=5$) can be found in \refap{sec:face.cls.supp}. \reft {tbl:time} reports the average training time used by each method on XM2VTS ($\gamma=4$) and PIE ($\gamma=6$) with outlying images. The main observations include

(i) Recognition: RFPCA and TPCA perform similarly and significantly outperform FPCA, BPCA and PCA in general. BPCA or PCA performs the worst while \emph{t}PCA fails to run as $N_{tr}<d$.

(ii) Robustness: RFPCA and TPCA are the most robust, as they perform stably in  recognition and the corresponding latent dimension, regardless of whether or not the dataset involves outliers. In contrast, the other methods are not robust because  when outliers are added, they usually either perform worse or require more number of features to achieve the optimal recognition performance.

(iii) Computational efficiency: TPCA requires much more time than the other methods.

In fact, we have also tried a variant of \emph{t}PCA by restricting the smallest eigenvalue of $\btSig$ in \refe{eqn:ecme.t.S} to be a small positive number $10^{-6}$. In this way, \emph{t}PCA can be run. However, its result is similar to that of PCA, and the normal observations and outliers have similar weights.

\subsection{Multivariate time series (MTS) datasets}\label{sec:expr.mts}
Multivariate time series (MTS) is a sort of matrix data and the observation matrix  typically consists of multiple variables measured at multiple time points, namely one mode in MTS is variable and the other is time. This is remarkably different from the image data in \refs{sec:cls.face.res}, where both modes of the data matrix are variables (i.e. pixels). Intuitionally, it seems that the assumption of separable covariance matrix in the four matrix-based methods are more suitable for such data. Therefore, in this experiment, we further use two publicly available real-world MTS datasets to examine the classification performance of the six methods. 

\textbullet\, The AUSLAN dataset\footnote{Available from \url{http://archive.ics.uci.edu/ml/datasets}} contains 2565 observations, captured from a native AUSLAN speaker using 22 sensors (i.e. 22 variables) on the CyberGlove. This dataset contains 95 signs (i.e. 95 classes), and each sign has 27 observations. In our experiment, we use a subset consisting of 25 signs (hence, 675 observations in total). These 25 signals are alive, all, boy, building, buy, cold, come, computer, cost, crazy, danger, deaf, different, girl, glove, go, God, joke, juice, man, where, which, yes, you and zero. The time length we use is 47 and hence the observation size is 47$\times$ 22.

\textbullet\, The ECG dataset\footnote{Available from \url{http://www.mustafabaydogan.com./files/viewcategory/20-data-sets.html}} contains 200 MTS observations, each is collected by two electrodes (i.e. 2 variables) during a heartbeat. The dataset contains two classes: normal and abnormal, and they have 133 and 67 observations, respectively. The time length we use is 39 and hence the observation size is 39$\times$ 2.

\subsubsection{Outlier detection}\label{sec:mts.odtc}
In this experiment, we use the two MTS datasets to further investigate the performance for outlier detection between RFPCA and \emph{t}PCA. However, the result by \emph{t}PCA is not available on AUSLAN and ECG ($p=1/4$), as \emph{t}PCA can not be run when the sample size $N_{tr}<d$. For example, $N_{tr}=605, d=1034$ on AUSLAN ($p=4/5$) and $N_{tr}=55, d=78$ for ECG ($p=1/4$). 

\reff{fig:mts.oidc} shows the scatter plot of the weights for all the normal and outlying matrix observations on AUSLAN and ECG datasets. It can be seen that RFPCA works well on both AUSLAN and ECG datasets while \emph{t}PCA can only be used on ECG dataset when $p=4/5$.

\begin{figure*}[htb]
	\centering \scalebox{0.8}[0.8]{\includegraphics*{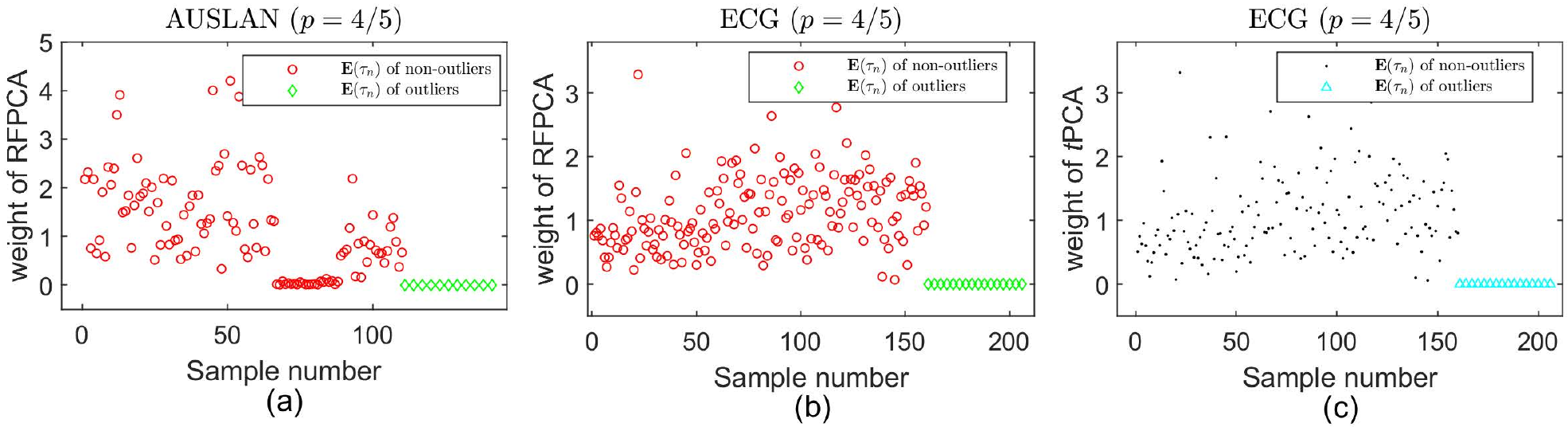}}
	\caption{Weights of non-outliers (circle) and outliers (diamond) by RFPCA on (a) AUSLAN and (b) ECG. Weights of non-outliers (point) and outliers (up triangle) by \emph{t}PCA on (c) ECG.} \label{fig:mts.oidc}
\end{figure*}

\subsubsection{Classification Performance on Multivariate time series (MTS) datasets}\label{sec:mts.cls}
To measure the classification error rate, we randomly select a part of the observations with proportion $p$ from each class as the training set, and the remaining observations as the test set. For each dataset, two values of $p$ from the candidate set $\{1/6,1/4,4/5\}$ are considered. Since the ECG dataset has a large numerical range, we preprocess the data by the normalization transformation. To examine the robustness, we add 10\% outlying observations from the uniform distribution $U(-10,10)$ to the training set. Like \refs{sec:cls.face}, we use the simple 1-nearest neighbor classifier in the reduced-dimensional space to obtain the classification error rates.
\begin{table*}[htb]
	\centering
	\caption{\label{tab:cls.MTS} The lowest average error rates (mean $\pm$ std) and their corresponding dimensions by different methods on the MTS datasets. The best method is shown in boldface. $\bullet$ means that RFPCA is significantly better than the method, and $\circ$ means that the difference between RFPCA and the method is not significant, using the Wilcoxon signed-rank test with a $p$-value of 0.05.}
	\resizebox{\linewidth}{!}{
		\begin{tabular}{ccccc}
			\toprule
			\multirow{2}{*}{Method} &   \multicolumn{2}{c}{AUSLAN ($p=4/5$)}      & \multicolumn{2}{c}{ECG ($p=4/5$)}  \\ \cmidrule(r){2-3} \cmidrule(r){4-5}
			&        without outliers      &     with 10\% outliers      &      without outliers       &    with 10\% outliers  \\ \midrule
			RFPCA      &5.2$\pm$1.5(1, 14)            & {\bf5.1}$\pm$1.5(1, 14)     &{\bf12.0}$\pm$3.9(11, 1)     &{\bf12.5}$\pm$6.8(8, 1)\\
			TPCA       &            —                 &           —                 &12.5$\pm$2.9(11, 1)$\circ$   &12.8$\pm$4.3(11, 1)$\circ$\\      
			FPCA       &{\bf2.8}$\pm$1.1(1, 14)$\circ$&7.4$\pm$4.1(1, 20)$\bullet$  &12.2$\pm$4.8(8, 1)$\circ$    &14.5$\pm$3.9(29, 2)$\circ$\\
			BPCA       &6.2$\pm$2.7(1, 21)$\bullet$   &21.8$\pm$2.8(21, 21)$\bullet$&15.8$\pm$5.3(21, 2)$\bullet$ &15.5$\pm$5.0(29, 2)$\bullet$\\
			\emph{t}PCA    &              —               &   —                         &13.5$\pm$5.4(10)$\circ$      &13.3$\pm$5.8(10)$\circ$\\
			PCA        &9.9$\pm$1.5(44)$\bullet$      &10.4$\pm$1.5(95)$\bullet$    &14.3$\pm$6.8(22)$\circ$      &15.8$\pm$5.9(25)$\circ$  \\ \bottomrule
		\end{tabular}
	}
\end{table*} 

\reft{tab:cls.MTS} summarizes the results on AUSLAN ($p=4/5$) and ECG ($p=4/5$) over 10 random splittings, where the `—' sign means that the method fails to run and hence the error rate is not available. The results on AUSLAN ($p=1/6$) and ECG ($p=1/4$) can be found in \refap{sec:mts.cls.supp}. The main observations are generally consistent with those in \refs{sec:cls.face.res}, but with the following three exceptions. 

(i) Although FPCA obtains the best performance on the original AUSLAN dataset, it degenerates significantly when outliers are added.

(ii) The only case that \emph{t}PCA can be run is ECG ($p=4/5$). Nevertheless, it performs slightly worse than RFPCA and TPCA.

(iii) No matter whether the AUSLAN dataset include outliers or not, both the TPCA  implementation by our program and that with the R package \texttt{MixMatrix} \citep{MixMatrix} fail to run due to numerical problems.

\section{Conclusion and future works}\label{sec:Conclusion} 
In this paper, we propose a new robust extension of FPCA (RFPCA), which is built upon the matrix-variate $t$ distribution. To obtain the ML estimates of the parameters, we develop an ECME algorithm and its parameter-expanded variant, namely PX-ECME. The PX-ECME is compared favorably with the ECME as shown in our experiments. Our empirical results on matrix datasets show that RFPCA improves the robustness of FPCA as expected. More importantly, RFPCA is much more robust than its vector-based counterpart \emph{t}PCA, which means that RFPCA significantly improves the capability of \emph{t}-type distributions in matrix data applications. In addition, both RFPCA and TPCA are robust, but RFPCA can be readily used for outlier detection and the proposed PX-ECME algorithm for RFPCA is more effective and efficient. 

Given our empirical result that the theoretical bound of the multivariate \emph{t} distribution is not applicable to RFPCA, for future research it is interesting to explore this issue theoretically, like that for the multivariate \emph{t} distribution \citep{Dumbgen-bkd}. On the other hand, mixture modeling using matrix-normal distributions have been proposed in \cite{viroli2011finite} for classifying and clustering matrix data, it would be worthwhile to extend our method in this paper to mixtures of matrix-variate $t$ distributions for outlier accommodation and detection, and compare with its vector-based counterpart, namely mixtures of multivariate $t$ distributions. The extension to the case in the presence of missing data can also be considered, similar to \cite{wang2015mixtures}. 

\begin{appendices}
\section{The Proofs}
\subsection{The PX-ECME algorithm as a modified ECME}\label{sec:PX-ECME.ECME}
Under \textsf{Model IV} in \refe{eqn:mvt.hrc}, we have $\bX_n\sim Mt_{c,r}(\bM, \bSig_c, \bSig_r, \nu)$. Under expanded \textsf{Model V} in \refe{eqn:mvt.hrc.PX}, we have $\bX_n\sim Mt_{c,r}(\bM_*, \bSig_{c*}/\alpha, \bSig_{r*}, \nu_*)$. It can be seen that $\bSig_{c*}/\alpha$ corresponds to $\bSig_c$ and $(\bM_*, \bSig_{r*}, \nu_*)$ to $(\bM,\bSig_r,\nu)$. Under \textsf{Model V}, the PX-ECME algorithm in the expanded parameter space proceeds as follows.

\noindent{\bf E-step:} The required expectation is computed by
\begin{equation*}
	\bbE[\tau_{n*}|\bX_n]=\alpha\frac{\nu_{*}+cr}{\nu_{*}+\tr\{(\bSig_{c*}/\alpha)^{-1}(\bX_n-\bM_{*})\bSig_{r*}^{-1}(\bX_n-\bM_{*})'\}}.
\end{equation*}
\noindent{\bf CM-steps:} The parameters $(\alpha,\btM_{*},\btSig_{c*},\btSig_{r*})$ are updated by
\begin{IEEEeqnarray*}{rCl}
	\alpha&=&\sum\nolimits_{n=1}^N\bbE[\tau_{n*}|\bX_n]/N,\\
	\btM_{*}&=&\frac1{\sum\nolimits_{n=1}^N\bbE[\tau_{n*}|\bX_n]}\sum\nolimits_{n=1}^N\bbE[\tau_{n*}|\bX_n]\bX_n,\\
	\btSig_{c*}&=&\frac{1}{Nr}\sum\nolimits_{n=1}^N\bbE[\tau_{n*}|\bX_n] (\bX_n-\btM_{*})\bSig_{r*}^{-1}(\bX_n-\btM_{*})',\\
	\btSig_{r*}&=&\frac{1}{Nc}\sum\nolimits_{n=1}^N\bbE[\tau_{n*}|\bX_n] (\bX_n-\btM_{*})'\btSig_{c*}^{-1}(\bX_n-\btM_{*}).	
\end{IEEEeqnarray*}
Since the updated degrees of freedom $\nu_{*}$ is achieved by maximizing the observed data log-likelihood $\cL$ given by \refe{eqn:tfpca.like}, the step to update  $\nu_{*}$ under \textsf{Model V} is the same as the CML-step 4 under \textsf{Model IV} in \refs{sec:RFPCA.ECME}.

Applying the mapping $(\bM_*,\bSig_{c*}/\alpha,\bSig_{r*},\nu_*) \rightarrow (\bM,\bSig_c,\bSig_r,\nu)$ from \textsf{Model V} to \textsf{Model IV}, we obtain
\begin{IEEEeqnarray*}{rCl}
	\btM&=&\btM_{*}=\frac1{\sum\nolimits_{n=1}^N\bbE[\tau_{n*}|\bX_n]}\sum\nolimits_{n=1}^N\bbE[\tau_{n*}|\bX_n]\bX_n=\frac1{\sum\nolimits_{n=1}^N\bbE[\tau_n|\bX_n]}\sum\nolimits_{n=1}^N\bbE[\tau_n|\bX_n]\bX_n,\\
	\btSig_c&=&\btSig_{c*}/\alpha=\frac{1}{r \sum\nolimits_{n=1}^N\bbE[\tau_{n*}|\bX_n]}\sum\nolimits_{n=1}^N\bbE[\tau_{n*}|\bX_n] (\bX_n-\btM_{*})\bSig_{r*}^{-1}(\bX_n-\btM_{*})'\\
	&=&\frac{1}{r \sum\nolimits_{n=1}^N\bbE[\tau_n|\bX_n]}\sum\nolimits_{n=1}^N\bbE[\tau_n|\bX_n] (\bX_n-\btM)\bSig_r^{-1}(\bX_n-\btM)',\\
	\btSig_r&=&\btSig_{r*}=\frac{1}{Nc}\sum\nolimits_{n=1}^N\bbE[\tau_{n*}|\bX_n] (\bX_n-\btM_{*})'(\alpha\btSig_c)^{-1}(\bX_n-\btM_{*})\\
	&=&\frac{1}{c \sum\nolimits_{n=1}^N\bbE[\tau_n|\bX_n]}\sum\nolimits_{n=1}^N\bbE[\tau_n|\bX_n] (\bX_n-\btM)'\btSig_c^{-1}(\bX_n-\btM).
\end{IEEEeqnarray*}
Therefore, the PX-ECME can be represented as a modified ECME except that $\btSig_c$ in \refe{eqn:ecme.Sc} and $\btSig_r$ in \refe{eqn:ecme.Sr} are replaced by \refe{eqn:px-ecme.Sc} and \refe{eqn:px-ecme.Sr}, respectively.

\subsection{Proof for \refp{prop:mvt.weight}}\label{sec:Prop1.proof}
\begin{proof}
	For ML estimate $\btheta$, multiplying \refe{eqn:ecme.Sc} by $\btSig_c^{-1}$ and taking the trace on both sides, we obtain 
	\begin{equation}\label{eqn:mt.cr}
		cr =  \frac{1}{N}\sum\nolimits_{n=1}^N\bbE[\tau_n|\bX_n]\tr\{\bSig_c^{-1}(\bX_n-\bM)\bSig_r^{-1}(\bX_n-\bM)'\}.	
	\end{equation}
	From \refe{eqn:Etau.X}, we have 
	\begin{equation}\label{eqn:Etau.X.trans}
		\nu+cr = \nu\bbE[\tau_n|\bX_n]+\bbE[\tau_n|\bX_n]\tr\{\bSig_c^{-1}(\bX_n-\bM)\bSig_r^{-1}(\bX_n-\bM)'\}.
	\end{equation}
	Taking the sum over $i$ from 1 to $N$ on both sides of \refe{eqn:Etau.X.trans} yields
	\begin{equation}\label{eqn:sum.Etau.X}
		N\nu + Ncr = \nu\sum\nolimits_{n=1}^N\bbE[\tau_n|\bX_n]+\sum\nolimits_{n=1}^N\bbE[\tau_n|\bX_n]\tr\{\bSig_c^{-1}(\bX_n-\bM)\bSig_r^{-1}(\bX_n-\bM)'\}.
	\end{equation}
	Substituting \refe{eqn:mt.cr} into \refe{eqn:sum.Etau.X}, we obtain
	\begin{equation*}
		\frac{1}{N}\sum\nolimits_{n=1}^N\bbE[\tau_n|\bX_n] = 1.
	\end{equation*}
	The proof is concluded.
\end{proof}

\section{Supplementary results on real data}
\subsection{Supplementary results on face datasets}\label{sec:face.cls.supp}
\reff{fig:oidc} and \reft{tab:cls.face} in \refs{sec:cls.face} have given the results for XM2VTS ($\gamma=4$) and PIE ($\gamma=6$). The results for XM2VTS ($\gamma=3$) and PIE ($\gamma=5$) are shown in \reff{fig:oidc.supp} and \reft{tab:cls.face.supp}.
\begin{figure*}[htb]
	\centering \scalebox{0.9}[0.9]{\includegraphics*{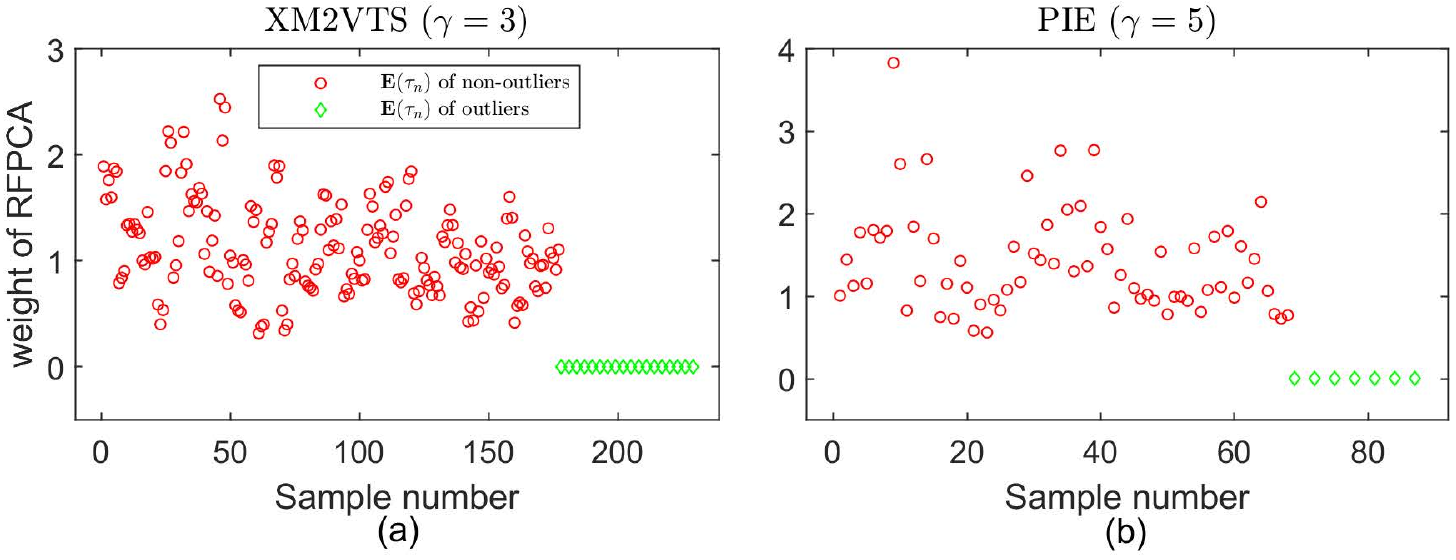}}
	\caption{Weights of non-outliers (circle) and outliers (diamond) by RFPCA on (a) XM2VTS and (b) PIE.} \label{fig:oidc.supp}
\end{figure*}
\begin{table*}[htbp]
	\centering
	\caption{\label{tab:cls.face.supp} The lowest average error rates (mean $\pm$ std) and their corresponding dimensions by different methods on the face datasets. The best method is shown in boldface. $\bullet$ means that RFPCA is significantly better than the method, and $\circ$ means that the difference between RFPCA and the method is not significant, using the Wilcoxon signed-rank test with a $p$-value of 0.05.}
	\resizebox{\linewidth}{!}{
		\begin{tabular}{ccccc}
			\toprule
			\multirow{2}{*}{Method} &  \multicolumn{2}{c}{XM2VTS ($\gamma=3$)}    & \multicolumn{2}{c}{PIE ($\gamma=5$)}         \\ \cmidrule(r){2-3} \cmidrule(r){4-5} 
			&        without outliers        &     with 10\% outliers      &      without outliers      &    with 10\% outliers  \\ \midrule
			RFPCA   &  {\bf15.0}$\pm$2.3(16, 7)      &  {\bf15.0}$\pm$2.3(19, 7)   &{\bf14.4}$\pm$8.1(25, 12)      & {\bf14.3}$\pm$8.0(25, 12) \\ 
			TPCA    &  15.3$\pm$2.4(12, 7)$\circ$    &15.5$\pm$2.4(19, 7)$\bullet$ & 14.4$\pm$8.3(27, 11)$\circ$   & 14.5$\pm$8.3(27, 11)$\circ$\\
			FPCA    &  15.6$\pm$2.3(17, 7)$\bullet$  &16.6$\pm$2.3(17, 11)$\bullet$& 14.9$\pm$8.8(31, 11)$\circ$   & 16.5$\pm$9.4(64, 17)$\bullet$\\
			BPCA    &  21.1$\pm$2.5(25, 51)$\bullet$ &21.2$\pm$2.4(40, 47)$\bullet$& 40.1$\pm$6.8(64, 64)$\bullet$ & 40.1$\pm$6.8(64, 64)$\bullet$ \\
			\emph{t}PCA&               —                &               —             &             —              &     —   \\
			PCA     &  19.3$\pm$2.4(82)$\bullet$     &19.3$\pm$2.1(171)$\bullet$   & 16.7$\pm$10.9(60)$\bullet$    & 16.6$\pm$10.9(94)$\bullet$   \\ \bottomrule
		\end{tabular}
	}
\end{table*} 

\subsection{Supplementary results on MTS datasets}\label{sec:mts.cls.supp}
\reff{fig:mts.oidc} and \reft{tab:cls.MTS} in \refs{sec:expr.mts} have given the results for AUSLAN ($p=4/5$) and ECG ($p=4/5$). The results for AUSLAN ($p=1/6$) and ECG ($p=1/4$) are shown in \reff{fig:mts.oidc.supp} and \reft{tab:cls.MTS.supp}.
\begin{figure*}[htb]
	\centering \scalebox{0.9}[0.85]{\includegraphics*{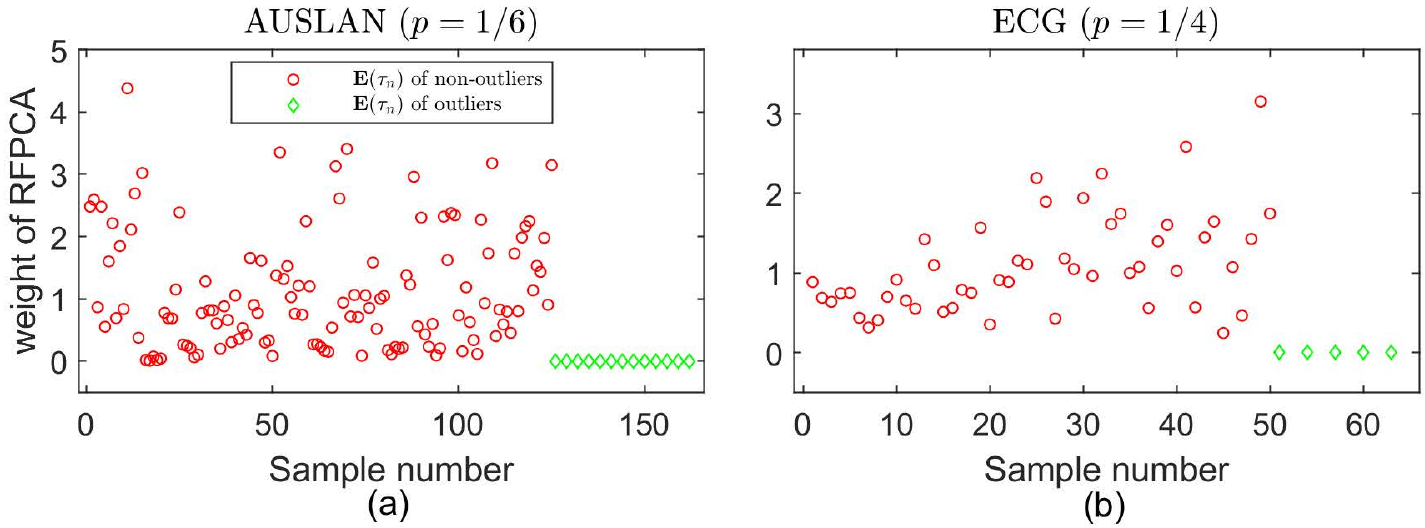}}
	\caption{Weights of non-outliers (circle) and outliers (diamond) by RFPCA on (a) AUSLAN and (b) ECG.} \label{fig:mts.oidc.supp}
\end{figure*}

\begin{table*}[htbp]
	\centering
	\caption{\label{tab:cls.MTS.supp} The lowest average error rates (mean $\pm$ std) and their corresponding dimensions by different methods on the MTS datasets. The best method is shown in boldface. $\bullet$ means that RFPCA is significantly better than the method, and $\circ$ means that the difference between RFPCA and the method is not significant, using the Wilcoxon signed-rank test with a $p$-value of 0.05.}
	\resizebox{\linewidth}{!}{
		\begin{tabular}{ccccc}
			\toprule
			\multirow{2}{*}{Method} &   \multicolumn{2}{c}{AUSLAN ($p=1/6$)}      & \multicolumn{2}{c}{ECG ($p=1/4$)}  \\ \cmidrule(r){2-3} \cmidrule(r){4-5}
			&        without outliers      &     with 10\% outliers      &      without outliers       &    with 10\% outliers  \\ \midrule
			RFPCA      &12.0$\pm$1.7(1, 15)           &{\bf11.4}$\pm$1.5(1, 14)     &{\bf18.4}$\pm$1.6(8, 1)      &{\bf18.1}$\pm$1.5(8, 1)\\
			TPCA       &        —                     &       —                     &18.5$\pm$1.9(8, 1)$\circ$    &19.0$\pm$1.6(8, 1)$\bullet$ \\
			FPCA       &{\bf8.8}$\pm$1.8(1, 15)$\circ$&16.9$\pm$2.1(3, 22)$\bullet$ &18.9$\pm$2.6(8, 1)$\circ$    &19.3$\pm$2.7(18, 1)$\circ$ \\
			BPCA       &12.5$\pm$1.6(1, 20)$\bullet$  &38.2$\pm$3.5(37, 22)$\bullet$&20.5$\pm$1.6(12, 1)$\bullet$ &21.2$\pm$2.3(25, 1)$\bullet$ \\
			\emph{t}PCA    &               —              &                  —          &              —              &   —  \\
			PCA        &21.5$\pm$2.0(36)$\bullet$     &21.6$\pm$2.1(49)$\bullet$    &20.1$\pm$2.3(10)$\bullet$    &19.6$\pm$2.9(20)$\circ$ \\ \bottomrule
		\end{tabular}
	}
\end{table*} 
\end{appendices}
\bibliographystyle{chicago}
\bibliography{journall,RFPCA,jhzhao-pub}
\end{document}